\documentclass{article}

\usepackage[margin=1in]{geometry}
\usepackage[numbers,sort&compress]{natbib}
\usepackage[utf8]{inputenc}
\usepackage[T1]{fontenc}
\usepackage{lmodern}
\usepackage{courier}  
\usepackage{amsmath,amsthm,amssymb}
\usepackage{algorithm}
\usepackage{algorithmic}
\usepackage{booktabs}
\usepackage{multirow}
\usepackage{makecell}
\usepackage{threeparttable}
\usepackage{siunitx}
\usepackage{graphicx}
\usepackage{colortbl}
\definecolor{bestcell}{RGB}{220,237,220} 
\graphicspath{{figures/}}
\usepackage{xcolor}
\usepackage{url}
\usepackage{hyperref}
\hypersetup{hidelinks,breaklinks=true}
\usepackage{cleveref}
\usepackage{subcaption}
\newcommand{\subfloat}[2][]{\subcaptionbox{#1}{#2}}
\usepackage{wrapfig}
\usepackage{enumitem}
\usepackage{pifont}
\usepackage{microtype}
\usepackage{nicefrac}
\usepackage{needspace}
\usepackage{tikz}
\usetikzlibrary{arrows.meta,positioning,shapes.geometric,calc}

\makeatletter

\def\TPT@opt@flushleft{%
  \def\TPTnoteSettings{%
    \setlength{\leftmargin}{0pt}%
    \setlength{\labelwidth}{0pt}%
    \setlength{\labelsep}{0pt}%
    \setlength{\itemindent}{0pt}%
    \setlength{\listparindent}{0pt}%
    \setlength{\rightmargin}{0pt}%
    \setlength{\rightskip}{0pt}%
    \setlength{\leftskip}{0pt}%
  }%
  \def\TPTnoteLabel##1{}%
  \rightskip\z@skip
  \leftskip\z@skip
}
\makeatother

\newtheorem{theorem}{Theorem}
\newtheorem{proposition}[theorem]{Proposition}
\newtheorem{lemma}[theorem]{Lemma}

\newtheorem{assumption}{Assumption}

\newenvironment{nolinenumbers}{}{}

\newcommand{\expect}{\mathbb{E}}
\newcommand{\fullkv}{\textsc{FullKV}}

\newcommand{\parhead}[1]{\noindent\textbf{#1}\hspace{0.25em}\ignorespaces}

\title{When Does Value-Aware KV Eviction Help? A Fixed-Contract Diagnostic for Non-Monotone Cache Compression}

\author{
  Ruijie Zhang, Haozhe Liang, Da Chang, Li Hu, Fanqi Kong, Huaxiao Yin, and Yu Li\\
  \texttt{jeffy353866@gmail.com}
}

\begin{document}

\maketitle
\begin{abstract}
Long-context LLM inference is bottlenecked by the memory and bandwidth cost of reading large KV caches during decoding. KV compression reduces this cost by keeping only part of the cache, but task accuracy alone does not identify why a selector succeeds or fails. A selector can fail at three steps: it may miss the evidence future decoding needs, give high scores to tokens that do not affect the output, or break related evidence when fitting scores into a small cache. We introduce a fixed-contract diagnostic that holds the selector's setup fixed and changes one decision slot at a time. For value ranking, the probe combines a block's attention mass with the estimated output change from removing it. On LongBench across three models and two budgets, the probe is positive on 72.6\% of positive-margin cells and 32.4\% of nonpositive-margin cells. NeedleBench M-RT at 32k and a RULER 8k check probe support closure under branched retrieval, and a 264-cell sign evaluation separates support recovery and output-value ranking from leverage effects near the boundary. The resulting order is to recover decode-side evidence, rank its output value, and preserve coupled evidence during projection.
\end{abstract}


\section{Introduction}
\label{sec:intro}

Long-context LLM inference moves the bottleneck from prefill to decoding. After prefill, every generated token must read an expanding key-value cache, so latency and memory traffic both rise with context length~\citep{zhang2023ho,li2024snapkv,xiao2024efficient}. KV compression methods for pretrained models reduce this cost by deciding which cache states to retain, read, or store at lower precision~\citep{li2024snapkv,feng2026adakv,qin2025cake,liu2024kivi}, and adjacent decode-aligned and sparse-readout systems modify the access estimator or read substrate itself~\citep{tian2026matters,ahn2026lookaheadkv,bai2026indexcache,yuan2025native}. Yet a fixed-budget task score mixes several failure modes (\Cref{fig:mechanism}). A selector may miss the evidence future decoding consumes, give high scores to tokens that do not affect the output, or break related evidence when fitting scores into a small cache.

\begin{figure}[t]
\centering
\includegraphics[width=0.86\linewidth]{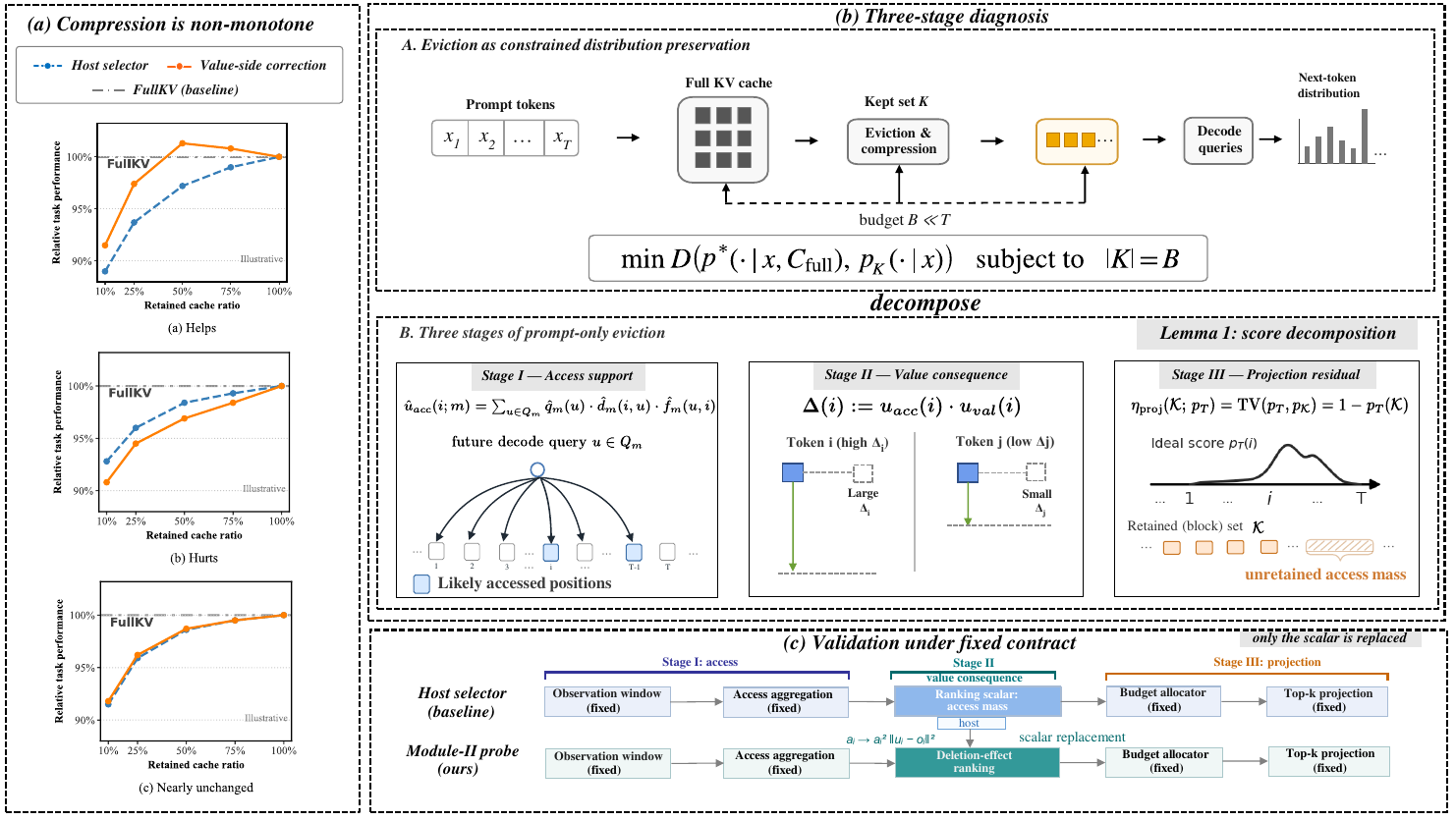}
\caption{Compression is not monotone. A compressed cache can underperform, match, or exceed \fullkv{} depending on which selector stage is the bottleneck. The diagnostic fixes the selector contract and perturbs one ranking slot to localize the failure.}
\label{fig:mechanism}
\end{figure}

Under tight budgets, performance differences between selectors concentrate at a small number of \emph{boundary units} near the budget threshold, while the bulk of the retained set is largely shared. Boundary-swap exposes this concentration. It keeps the model, budget, and all other cache positions fixed, exchanges only the decisions near the threshold, and recovers most of the gap between selectors. At low budget, boundary-rank correctness is the binding constraint.

These boundary errors map to the three selector stages. In access estimation, \emph{query-domain mismatch} appears when historical prefill statistics deviate from the decode queries that actually consume the cache, so early positions receive inflated scores simply from being seen by more subsequent queries. In value-consequence ranking, \emph{output-unaware scoring} treats attention mass as evidence of access without testing whether deletion would change the output, so task-relevant evidence ends up mixed with templates and delimiters. In score-to-cache projection, \emph{projection mismatch} appears when tokenwise ranking under multi-evidence decoding scatters budget across isolated high-score tokens, so evidence coverage stays incomplete and continuous support regions fragment.

The diagnostic fixes the prefill-attention tensor, query domain, observation window, budget, allocation rule, and score-to-cache projection before outcome measurement. A stage-local test changes one slot in that contract. The main test keeps the SnapKV~\citep{li2024snapkv} observation-window contract fixed and substitutes a value-consequence ranking scalar into the same ranking slot to test whether boundary units contain recoverable value-consequence misranking. Cross-contract rows on PyramidKV~\citep{cai2024pyramidkv}, CAKE~\citep{qin2025cake}, Ada-KV~\citep{feng2026adakv}, and H2O~\citep{zhang2023ho} are transfer and sensitivity checks that show how the sign pattern changes when access estimation, allocation, or projection also moves.

Under the fixed selector contract, each experiment supports only the stage whose slot is changed. Boundary swaps first localize tight-budget errors to the score margin and separate query-domain mismatch, output-unaware scoring, and projection mismatch. With the contract fixed, LongBench~\citep{bai2024longbench} controls test the predicted sign split before reading outcomes. The value-consequence intervention helps mainly in cells with positive reference margin and low support coupling, while no-value controls collapse this separation. Token-fill makes unused budget an unlikely standalone Stage~III explanation under the fixed block control, and cross-contract baselines test transfer without being used to rank selectors. NeedleBench~\citep{lineedlebench} and the RULER~\citep{hsieh2024ruler} 8k check probe support closure under branched retrieval targets. A wider 264-cell sign evaluation shows high agreement between no-leverage and final signs, separating support recovery and output-value ranking from leverage as a conditional amplifier.

\section{Staged Diagnostic Framework}
\label{sec:observation}
\label{sec:diagnosis}
\label{sec:method_main}

\parhead{Setup.}
\label{sec:background}
Under tight budgets, a prefill-observable score can fail by estimating the wrong future access support, mis-ranking useful positions, or losing support during budget projection. The pipeline is formalized under a fixed retained-token budget, with notation collected in \Cref{sec:app_notation}.

Let $\mathbf{A}_{l,h}$ denote the layer-head prefill attention matrix, with scalar entry $\mathbf{A}_{l,h}[u,i]$ from query position $u$ to key position $i$, where $T$ is the prompt length and $b\in(0,1)$ the budget ratio. An eviction rule first constructs scalar scores from the observable prefill-attention tensor $\mathbf{A}$ and then projects them to a fixed memory budget,
\begin{equation}
s_i = \psi_{\mathrm{score}}(\mathbf{A}, i), \qquad \mathcal K = \operatorname{TopK}(s, k), \quad k = \lfloor bT \rfloor \le T.
\label{eq:eviction}
\end{equation}
Write $a_u(i):=\operatorname{Agg}_{l,h}\,\mathbf{A}_{l,h}[u,i]$ when head and layer aggregation are not central, and keep tensor form only where cross-layer structure matters. $\operatorname{TopK}$ uses a score-independent deterministic tie break.

\begin{assumption}[Fixed selector contract]
\label{assump:fixed_contract}
Diagnostic comparisons fix the prefill-attention tensor, query or proxy domain, causal mask, observation window, budget $k=\lfloor bT\rfloor$, and projection rule before outcome measurement. Query laws share a zero-extended finite domain, and blockwise variants state the block size and boundary policy.
\end{assumption}
All formal statements in this section use \Cref{assump:fixed_contract} unless a different contract is stated.

For a fixed selector contract and prompt position $i$, write $u_i$ for the latent decode-side utility of retaining $i$, $e_i$ for the access-support exposure factor estimated by $u_{\mathrm{acc}}$, $g_i$ for token-level layer, head, or budget scaling, and $\rho_i$ for structured non-utility mass such as delimiters, records, or format markers. Any scalar prefill score can be decomposed as
\begin{equation}
s_i = e_i\,g_i\,u_i + \rho_i + \xi_i,
\label{eq:score_observation_model}
\end{equation}
with residual $\xi_i:=s_i-e_i g_i u_i-\rho_i$. This representation is definitional. Its role is to mark which factor a stage-local intervention changes. The support-coupling index $\phi(x)$ flags $\rho_i$-type structure, while $\eta_{\mathrm{proj}}$ is defined only after projection. The fixed-predictor split in \Cref{sec:exp_module2} tests Stage~II under the SnapKV~\citep{li2024snapkv} contract by holding access estimation and projection fixed.

\parhead{Why a staged diagnostic is needed.}
Under low budgets, most prompt positions are either safely retained or safely evicted, and the outcome is decided by positions near the keep-or-evict boundary. Boundary swaps recover much of the gap between collapsed cumulative scorers and decode-aligned scorers, whereas random swaps of the same size do not. This locality rules out an explanation based only on global score calibration. The active error is tied to how a selector observes future access, assigns value to accessed states, or projects a score into the finite cache.

\begin{lemma}[Boundary-margin condition]
\label{lem:boundary_margin}
Fix a selector contract with deterministic tie breaking, a base score $s$, and a perturbed score $s'=s+\delta$. Let $j\in\mathcal K_s$ be a retained token and $i\notin\mathcal K_s$ an evicted token under the base kept set $\mathcal K_s=\operatorname{TopK}(s,k)$. The perturbation strictly ranks $i$ above $j$ whenever
\begin{equation}
\delta_i-\delta_j > s_j-s_i.
\label{eq:boundary_margin_condition}
\end{equation}
At equality, the fixed tie break decides. If the downstream utility is locally additive for this one-for-one boundary swap, then replacing $j$ by $i$ changes latent kept-set utility by $u_i-u_j$. Away from the equality case, a score perturbation is beneficial at this boundary when it crosses the base margin and its signed score change agrees with the latent utility gap.
\end{lemma}
The proof appears in \Cref{sec:proof_boundary_margin}. This local condition motivates the sign-resolved groups used in the experiments.

\label{sec:diagnostic_decomp}
The diagnostic targets are decode-side access support $u_{\mathrm{acc}}(i)$, value consequence $u_{\mathrm{val}}(i)$ among accessed positions, and projection residual $\eta_{\mathrm{proj}}$ after mapping a score to a fixed token, head, layer, or block budget. They are instantiated by \Cref{eq:access_support_estimators,eq:stageii_first_order,eq:projection_residual}. \Cref{eq:score_observation_model,prop:ordered_substitution,prop:last_prefill_anchor} give the corresponding algebraic checks.

\begin{equation}
\hat u_{\mathrm{acc}}(i;m)
:=
\sum_{u\in\mathcal Q_m} \hat q_m(u)\,\hat d_m(i,u)\,\hat f_m(u,i),
\qquad
u_{\mathrm{acc}}(i)
:=
\sum_{u\in\mathcal Q_m} q^\star(u)\,d^\star(i,u)\,f^\star(u,i),
\label{eq:access_support_estimators}
\end{equation}
Here $\mathcal Q_m$ is the zero-extended causal query domain. Estimator $m$ uses query law $\hat q_m$, exposure correction $\hat d_m$, and pooling kernel $\hat f_m(u,i):=\sum_{l,h}\hat\beta_{l,h,m}(u)\,\mathbf{A}_{l,h}[u,i]$. Starred terms denote the references.

The Stage~I access error then decomposes exactly into query-law, exposure, and aggregation terms.

\begin{proposition}[Ordered-substitution identity]
\label{prop:ordered_substitution}
Under \Cref{assump:fixed_contract}, with estimators defined on the common query domain in \Cref{eq:access_support_estimators}, the access-support estimation error admits the exact decomposition
\begin{equation}
\hat u_{\mathrm{acc}}(i;m)-u_{\mathrm{acc}}(i)
=
\delta_{\mathrm{phase}}(i;m)
+
\delta_{\mathrm{exp}}(i;m)
+
\xi_{\mathrm{acc}}(i;m),
\label{eq:decomp}
\end{equation}
where
\begin{align*}
\delta_{\mathrm{phase}}(i;m) &:= \textstyle\sum_{u\in\mathcal Q_m} \bigl(\hat q_m(u)-q^\star(u)\bigr)\,d^\star(i,u)\,f^\star(u,i), \\
\delta_{\mathrm{exp}}(i;m)   &:= \textstyle\sum_{u\in\mathcal Q_m} \hat q_m(u)\bigl(\hat d_m(i,u)-d^\star(i,u)\bigr)\,f^\star(u,i), \\
\xi_{\mathrm{acc}}(i;m)      &:= \textstyle\sum_{u\in\mathcal Q_m} \hat q_m(u)\,\hat d_m(i,u)\bigl(\hat f_m(u,i)-f^\star(u,i)\bigr).
\end{align*}
\end{proposition}
\emph{Proof in \Cref{sec:proof_ordered_substitution}.} The three terms isolate query-law mismatch, exposure correction, and the layer and head pooling residual under the fixed-order substitution above.

Stage~II reweights accessed positions by their conditional output consequence $u_{\mathrm{val}}(i)$. Define
\begin{equation*}
\begin{aligned}
\Delta(i) &:= u_{\mathrm{acc}}(i)\,u_{\mathrm{val}}(i),\\
\hat\Delta(i;m) &:= \hat u_{\mathrm{acc}}(i;m)\,\hat u_{\mathrm{val}}(i;m),\\
\xi_{\mathrm{val}}(i;m) &:= \hat u_{\mathrm{val}}(i;m)-u_{\mathrm{val}}(i).
\end{aligned}
\end{equation*}
Expanding the product gives
\begin{subequations}\label{eq:stageii_first_order}
\begin{align}
\hat\Delta(i;m) - \Delta(i)
&=
u_{\mathrm{val}}(i)\,\bigl[\delta_{\mathrm{phase}}(i;m)+\delta_{\mathrm{exp}}(i;m)+\xi_{\mathrm{acc}}(i;m)\bigr]
\notag\\
&\quad+
u_{\mathrm{acc}}(i)\,\xi_{\mathrm{val}}(i;m)
+ R_2(i;m),
\label{eq:stageii_first_order:a}\\
R_2(i;m)
&:=
\bigl[\delta_{\mathrm{phase}}(i;m)+\delta_{\mathrm{exp}}(i;m)
\notag\\
&\quad+
\xi_{\mathrm{acc}}(i;m)\bigr]\xi_{\mathrm{val}}(i;m).
\label{eq:stageii_first_order:b}
\end{align}
\end{subequations}
The first-order approximation keeps the linear terms and drops only the residual product. The SnapKV~\citep{li2024snapkv} replacement fixes access estimation and projection, perturbing only the value-consequence scalar and isolating the $u_{\mathrm{acc}}\,\xi_{\mathrm{val}}$ channel.

Stage~III projects $\hat\Delta$ onto a fixed budget. For a kept set $\mathcal K$ and an access law $p_T$ induced by the scoring queries, with $p_T(\mathcal K)>0$, the projection term is directly visible as retained access mass:
\begin{equation}
\eta_{\mathrm{proj}}(\mathcal K;p_T)
:=
\operatorname{TV}\!\left(p_T,\bar p_{\mathcal K}\right)
=
1 - p_T(\mathcal K),
\label{eq:projection_residual}
\end{equation}
where $\bar p_{\mathcal K}$ is $p_T$ restricted to $\mathcal K$ and renormalized, with residual $1$ if $p_T(\mathcal K)=0$. This grows when isolated high-score tokens miss support closure for code, records, or multi-target retrieval. It differs from the block-lattice deficit $\varepsilon_{\mathrm{lat}}$ in \Cref{sec:pap_topk}, which counts unused nominal budget.

\begin{proposition}[Tail-$K$ TV bound]
\label{prop:last_prefill_anchor}
Let $p^\star$ denote the access law of the first decode step, $p_T$ the access law of the last prefill query, $\bar p_{\mathrm{tail},K}$ the law of a tail-$K$ query mixture, and $\bar p_{\mathrm{all}}$ the law of the full prefill mixture. Then
\begin{align}
\operatorname{TV}(p^\star,\bar p_{\mathrm{tail},K}) &\le \operatorname{TV}(p^\star,p_T) + \operatorname{TV}(p_T,\bar p_{\mathrm{tail},K}), \label{eq:tailk_tv_bound}\\
\operatorname{TV}(p^\star,\bar p_{\mathrm{all}}) &\le \operatorname{TV}(p^\star,p_T) + \operatorname{TV}(p_T,\bar p_{\mathrm{all}}). \label{eq:allprefill_tv_bound}
\end{align}
\end{proposition}
\emph{Proof in \Cref{sec:proof_tailk_tv}.} The first term is shared across estimators, while the second is measurable from prefill attention. The bound does not prove decode alignment; it defines the measurable proxy gap used to compare tail-$K$ and detector-based query proxies against $p_T$.

\label{sec:exposure_bias}%
\label{sec:phase_dilution}%
\label{sec:diagnostic}%
\label{sec:confound_decomp}%
\label{sec:boundary_swap_main}%
\label{sec:niah}%
\noindent\textbf{Two coupled access-support distortions.}\nobreak\hspace{0.25em}Exposure Bias is a count distortion. Earlier positions accumulate larger cumulative attention because more causal prefill queries can see them~\citep{gu2025ahakv,zhao2025attentiondebiasing}. Count-debiasing lifts cumulative scoring, but can expose a reverse-depth failure on needle retrieval, where the kept set drifts toward merely recent positions. Appendix~\ref{sec:appendix_exposure_bias_theory} gives the full analysis.

Phase Dilution is a query-law distortion. Eviction serves future decode queries, whereas pooled prefill scorers average over formatting-heavy prompt queries. Tail-window and lookahead-style methods~\citep{tian2026matters,ahn2026lookaheadkv,wang2025lookahead} reduce this mismatch. Removing prefix over-counting exposes the tailward mismatch that raw cumulative scoring had partially masked, and boundary swaps confirm it by recovering most of the gap between H2O~\citep{zhang2023ho} and SnapKV~\citep{li2024snapkv} on disagreement-boundary units.

\parhead{Architecture-conditioned residual.}
\label{sec:residual}
After support-estimation repair, $\xi_{\mathrm{acc}}(i;m)$ becomes a conservative remainder. On LongBench~\citep{bai2024longbench} the decode-aligned family collapses into a narrow band, so this residual is small relative to the gap between cumulative and decode-aligned support. Head-aggregation dilution is one instance, with details in Appendix~\ref{sec:appendix_head_aggregation}.

\section{Fixed-Contract Diagnostic}
\label{sec:pap_topk}

A scalar eviction score entangles future-query approximation, value consequence, and score projection. Stage-local claims require fixing the selector contract before outcome measurement.

\parhead{Selector contract.}
A contract specifies the prefill attention tensor, query or proxy domain, causal mask, observation window, budget, allocation rule, and final projection. Two scores are comparable as a Stage~II intervention only if they share this contract and differ only in the value-consequence ranking slot. Comparisons that change the query law, layer or head budget, or projection rule are transfer and sensitivity checks, not clean Stage~II evidence.

\parhead{Stage~I access support.}
The clean experiment uses SnapKV's~\citep{li2024snapkv} observation-window access proxy, whose explicit query law, per-head allocation, deterministic top-$k$ projection, and scalar ranking slot make the Stage~II replacement identifiable. H2O-style cumulative scoring~\citep{zhang2023ho} and count-debiased variants change the access law and diagnose Stage~I instead.

Blocks are the selection and projection units for original KV states. Block centroids are only score features for deletion consequence.

Let the prompt length be $T$ and the token budget be $k=\lfloor bT \rfloor$. Cache positions are partitioned into contiguous blocks
\[
\mathcal{B}=\{c_1,\dots,c_{N_b}\},\qquad |c_j|=p \ \text{for}\ j<N_b,\qquad |c_{N_b}|\le p,
\qquad
k_b=\left\lfloor \frac{k}{p}\right\rfloor .
\]
The construction scores a small selected layer set $\mathcal{L}_\star$ with layer weights $\beta_l$, typically drawn from late or plateau layers, and uses a proxy bank
\[
\mathcal{U} = \mathcal{U}^{\mathrm{tail}} \cup \mathcal{U}^{\mathrm{anchor}},
\qquad
\mathcal{U} = \bigsqcup_{m=1}^{M} \mathcal{U}_m,
\]
where $\mathcal{U}^{\mathrm{tail}}$ contains prompt-tail samples and $\mathcal{U}^{\mathrm{anchor}}$ contains instruction, question, or field-boundary anchors. When no explicit multi-target structure exists, $M=1$. Recency weighting is applied only inside the tail stratum,
\begin{equation}
\tilde r_u =
\begin{cases}
\exp\!\left(-\dfrac{T-u}{\tau_q}\right), & u\in \mathcal{U}^{\mathrm{tail}},\\[6pt]
1, & u\in \mathcal{U}^{\mathrm{anchor}},
\end{cases}
\qquad
r_u=\dfrac{\tilde r_u}{\sum_{v\in \mathcal{U}}\tilde r_v}.
\label{eq:pap_principle}
\end{equation}
These are the scoring objects for the value-channel replacement and projection controls, and the test keeps the host observation window and projection unchanged.

\parhead{Stage~II value-consequence replacement.}
Let $\kappa(h)$ map each query head to its KV head. For selected layer $l\in\mathcal{L}_\star$, proxy $u\in \mathcal{U}$, and query head $h$, the full attention output is
\[
\mathbf{o}_{l,h,u}=\sum_{i=1}^{T} \mathbf{A}_{l,h}[u,i]\,\mathbf{V}_{l,\kappa(h),i}.
\]
For a block $c\in\mathcal{B}$, define the block attention mass and attention-weighted value centroid as
\[
a_{l,h,u}(c)=\sum_{i\in c} \mathbf{A}_{l,h}[u,i],
\qquad
\boldsymbol{\mu}_{l,h,u}(c)=
\frac{\sum_{i\in c} \mathbf{A}_{l,h}[u,i]\,\mathbf{V}_{l,\kappa(h),i}}
{\max\{a_{l,h,u}(c),\varepsilon_\mu\}}.
\]
The Stage~II replacement uses deletion consequence as the ranking signal for already supported regions. With $\gamma_{\varepsilon,l,h,u}(c):=\max\{1-a_{l,h,u}(c),\varepsilon_a\}$, the surrogate is
\[
d_{l,h,u}(c)=
\underbrace{\left(
\frac{a_{l,h,u}(c)}{\gamma_{\varepsilon,l,h,u}(c)}
\right)^2}_{\text{support strength}}
\underbrace{\left\|
\boldsymbol{\mu}_{l,h,u}(c)-\mathbf{o}_{l,h,u}
\right\|_2^2}_{\text{conditional output consequence}}.
\]
The support factor matters only after the access proxy is placed on the relevant region, and clipping avoids the singular limit as $a$ approaches $1$. Aggregating over proxy groups gives
\[
D_m(c)=
\sum_{u\in \mathcal{U}_m} r_u
\sum_{l\in\mathcal{L}_\star}\beta_l
\sum_h d_{l,h,u}(c).
\]
The main-text score then pools these group-wise deletion costs,
\begin{equation}
S(c)=\sum_{m=1}^{M} w_m D_m(c),
\qquad
M=1,\; w_m=1 \text{ by default.}
\label{eq:pap_moduleii_score}
\end{equation}
The controlled replacement uses \Cref{eq:pap_moduleii_score} in the same ranking slot as the host selector, while the observation window, budget, allocation, and projection are unchanged. Robust aggregation and multi-target reserves are appendix extensions.

\parhead{Stage~III projection controls.}
Stage~III uses a block-constant posterior approximation $\hat\pi_i=\hat\pi_{b(i)}$, so the strict default form keeps the top $k_b$ blocks under the block budget:
\begin{equation}
\mathcal{Z}^\star = \operatorname{TopK}\bigl(\{S(c)\}_{c\in\mathcal B},\,k_b\bigr),
\qquad
\mathcal K = \bigcup_{c\in \mathcal{Z}^\star} c.
\label{eq:pap_block_topk}
\end{equation}
\Cref{eq:pap_block_topk} is the default Stage~III projection. It replaces the token-level keep posterior $\pi_i:=\Pr(z_i{=}1\mid\mathbf{A},k)$ by a block-constant proxy and projects under $k_b=\lfloor k/p\rfloor$. Strict block TopK leaves a lattice residual $\varepsilon_{\mathrm{lat}}(k,p):=k-pk_b$, distinct from the access-mass residual $\eta_{\mathrm{proj}}(\mathcal K;p_T)$ in \Cref{eq:projection_residual}. Token-fill removes the lattice residual without changing block order.

\parhead{Stage roles.}
The clean Stage~II test fixes access support and projection and changes only the value-consequence ranking slot. Block TopK, token-fill, grouped proxy reserves, and multi-target variants are projection controls. Cost and lattice variants are in \Cref{sec:app_algorithm}.

\section{Experiments}
\label{sec:experiments}

The experiments test the predicted sign of the value-consequence channel. The host selector, model, task, budget, sample set, observation window, allocation, and projection are fixed, and only the ranking scalar changes. Cross-contract baselines mark where attribution no longer holds. Appendix~\ref{sec:app_stage_registry} gives the full stage-to-evidence map and contract coverage.

\parhead{Setup.}
\label{sec:setup}
We evaluate Qwen3-8B~\citep{yang2025qwen3}, Llama-3.1-8B-Instruct~\citep{grattafiori2024llama}, and Mistral-7B-Instruct-v0.3~\citep{jiang2024mistral} on English LongBench~\citep{bai2024longbench} at $b\in\{0.05,0.10\}$, using 100 samples per task-model-budget cell and an 8192-token cap. Runs used eight NVIDIA A100-80GB GPUs. The 16 tasks, three models, and two budgets give a 96-cell grid. In every cell, \fullkv{}, SnapKV, the Module~II replacement (MII), and NoLev use the same model, task, budget, precision, context cap, and sample set. NoLev is the no-leverage MII variant: it keeps MII's value-consequence ranking scalar and the fixed selector contract, but replaces the leverage multiplier $(1-a)^{-2}$ with one.

We implement the observation-window contract with SnapKV~\citep{li2024snapkv} because its prefill query law, per-head allocation, and deterministic projection are explicit and leave one scalar ranking slot exposed.

\parhead{Stage~II evaluation.}
\label{sec:exp_module2}
Under this contract, the intervention changes only the value-consequence ranking scalar. MII replaces the identity score with the value-consequence score, and NoLev is the same replacement without the leverage multiplier. With access and projection fixed, \Cref{lem:boundary_margin} predicts lift only where the boundary still contains recoverable value-ranking error.

At block level, \Cref{fig:prefill_block_shift} shows how the Stage~II replacement reallocates retained mass without changing the observation-window frame. The difference map identifies the prefill blocks whose retained fraction diverges from SnapKV under the same selector contract.

\par\vspace{0.10em}
\begin{nolinenumbers}
\noindent\begin{minipage}{\textwidth}
\centering
\captionsetup{type=figure,hypcap=false,skip=2pt}
\setlength{\tabcolsep}{2pt}
\begin{tabular}{@{}c c c c c@{}}
\begin{minipage}[c]{0.27\linewidth}
\centering
\includegraphics[height=0.92in]{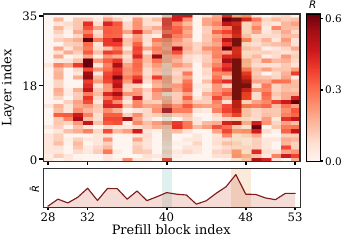}\\[-0.2em]
{\footnotesize (a) SnapKV mass.}
\end{minipage}
&
\begin{minipage}[c]{0.075\linewidth}
\centering
\begin{tikzpicture}
\path[use as bounding box] (-0.20in,-0.13in) rectangle (0.16in,0.13in);
\coordinate (arrstart) at (-0.20in,0);
\coordinate (arrend) at (0.16in,0);
\node[font=\scriptsize,inner sep=1pt,overlay] at (-0.02in,0.075in) {\makecell[c]{value channel}};
\draw[-{Latex[length=2mm,width=1.5mm]},thick] (arrstart) -- (arrend);
\node[font=\scriptsize,inner sep=1pt,overlay] at (-0.02in,-0.075in) {\makecell[c]{re-rank}};
\end{tikzpicture}
\end{minipage}
&
\begin{minipage}[c]{0.27\linewidth}
\centering
\includegraphics[height=0.92in]{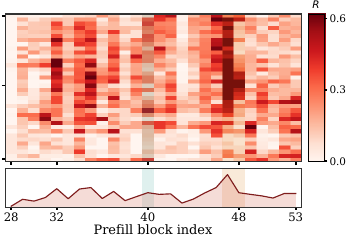}\\[-0.2em]
{\footnotesize (b) Stage~II mass.}
\end{minipage}
&
\begin{minipage}[c]{0.075\linewidth}
\centering
\begin{tikzpicture}
\path[use as bounding box] (-0.18in,-0.13in) rectangle (0.18in,0.13in);
\coordinate (arrstart) at (-0.18in,0);
\coordinate (arrend) at (0.18in,0);
\node[font=\scriptsize,inner sep=1pt,overlay] at (0,0.075in) {\makecell[c]{subtract}};
\draw[-{Latex[length=2mm,width=1.5mm]},thick] (arrstart) -- (arrend);
\node[font=\scriptsize,inner sep=1pt,overlay] at (0,-0.075in) {\makecell[c]{\phantom{re-rank}}};
\end{tikzpicture}
\end{minipage}
&
\begin{minipage}[c]{0.27\linewidth}
\centering
\includegraphics[height=0.92in]{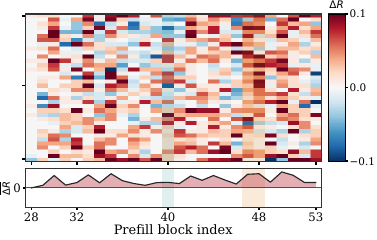}\\[-0.2em]
{\footnotesize (c) Difference map $\Delta R$.}
\end{minipage}
\end{tabular}
\caption{Stage~II value-channel re-ranking shifts retained mass across prefill blocks (Qwen3-8B, HotpotQA, $b=0.05$). Color denotes retained-token fraction in (a,b) and signed retained-mass change in (c); lower traces average the heatmap over layers.}
\label{fig:prefill_block_shift}
\end{minipage}
\end{nolinenumbers}
\par\vspace{0.05em}

The grid-level test groups cells by support coupling $\phi$ and signed reference-margin side $H_c$, both fixed before outcome measurement. The signed reference margin is $m_c=\fullkv{}_c-\mathrm{Host}_c$, and $H_c=\operatorname{sign}(m_c)$ gives the positive or nonpositive split. $\mathrm{Host}_c$ is the identity selector under the fixed contract for cell $c$. The support-coupling score $\phi$ identifies structured-support inputs whose evidence is distributed across dependent states. \Cref{fig:stageii_conditional} puts the fixed predictor and measured Stage~II shift on the same grid. Positive margin aligns with positive shift, and removing the value-consequence channel collapses the favorable-cell lift. Aggregated over the grid, the replacement is positive on $72.6\%$ of positive-margin cells and $32.4\%$ of nonpositive-margin cells ($2.2{\times}$). Within positive margin, low coupling gives a $76.4\%$ positive-shift rate at $+0.69$\,pp, $1.8{\times}$ the high-coupling rate ($42.9\%$, $-0.59$\,pp). \Cref{tab:mii_headroom} reports the signed reference-margin rates. Appendix~\ref{sec:app_prereg_provenance} gives the full $\phi\times H_c$ grid.

\par\noindent\begin{minipage}{\textwidth}
\captionsetup{type=figure,hypcap=false,skip=3pt}
\centering
\begin{minipage}[t]{0.335\textwidth}
\centering
\includegraphics[height=0.86in]{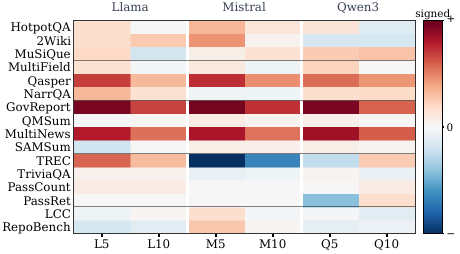}\\
{\footnotesize (a) Predicted signed reference margin.}
\end{minipage}\hfill
\begin{minipage}[t]{0.335\textwidth}
\centering
\includegraphics[height=0.86in]{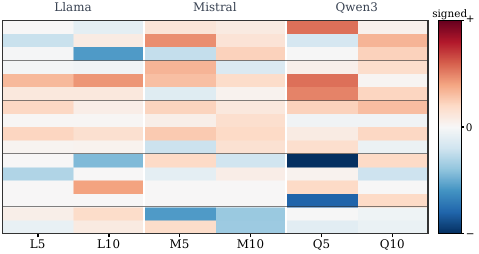}\\
{\footnotesize (b) Measured Stage~II lift.}
\end{minipage}\hfill
\begin{minipage}[t]{0.245\textwidth}
\centering
\includegraphics[height=0.86in]{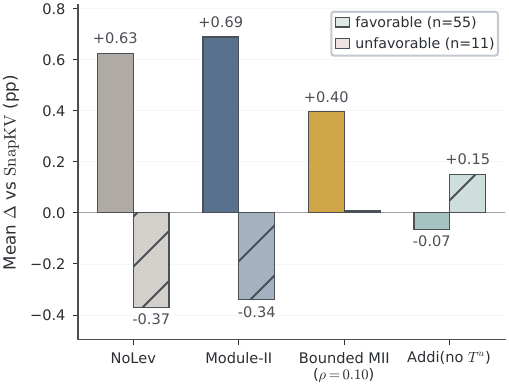}\\
{\footnotesize (c) Lift collapses without value channel.}
\end{minipage}
\caption{Stage~II prediction and value-channel control under a fixed observation-window contract. Panels (a) and (b) share rows for LongBench~\citep{bai2024longbench} tasks and columns for model-budget pairs. Heatmap color is signed percentage-point scale after normalization. Panel (a) is fixed before outcome measurement, panel (b) is the measured Stage~II shift, and panel (c) removes $T^u$ from the ranking scalar.}
\label{fig:stageii_conditional}
\end{minipage}
\par\vspace{0.2em}

The value-channel swap helps when the boundary still contains value-ranking error and weakens when the identity branch is already favorable or structurally incomplete. Nonpositive-margin cells leave no recoverable Stage~II room under the host reference, or point to a Stage~I or Stage~III bottleneck under the fixed contract.

\par\nopagebreak\vspace{0.05em}
\begin{nolinenumbers}
\noindent\begin{minipage}{\textwidth}
\centering
\captionsetup{type=table,hypcap=false,skip=2pt}
\caption{Signed reference margin separates positive from nonpositive Module~II shifts under the fixed SnapKV contract.}
\label{tab:mii_headroom}
\scriptsize
\setlength{\tabcolsep}{2pt}
\renewcommand{\arraystretch}{1.05}
\begin{tabular*}{\textwidth}{@{\extracolsep{\fill}}l l S[table-format=2.0] S[table-format=2.1] S[table-format=2.1] S[table-format=2.1] S[table-format=2.1] S[table-format=+1.2] l@{}}
\toprule
& & & \multicolumn{4}{c}{\textbf{Directional match (\%)}} & & \\
\cmidrule(lr){4-7}
\textbf{Margin side} & \textbf{Pred.} & \textbf{Cells} & \textbf{All} & \textbf{Llama} & \textbf{Mistral} & \textbf{Qwen3} & \textbf{Mean \(\Delta\) (pp)} & \textbf{\(p\)} \\
\midrule
$H_c>0$ ($m_c>0$) & $\Delta > 0$ & 62 & 72.6 & 61.9 & 80.0 & 76.2 & +0.54 & $2.5{\times}10^{-4}$ \\
$H_c\le0$ ($m_c\le0$) & $\Delta \le 0$ & 34 & 67.6 & 54.5 & 75.0 & 72.7 & -0.44 & $2.9{\times}10^{-2}$ \\
\bottomrule
\end{tabular*}
\end{minipage}
\end{nolinenumbers}
\par\vspace{0.10em}

\parhead{Falsifiers and robustness.}
\label{sec:exp_confirmatory}
All confirmatory splits are fixed before the outcome is read. The support coupling $\phi(x)$ is a structure-coupling grammar score; $H_c=\operatorname{sign}(m_c)$ is the signed reference-margin split. Support coupling is task-level, while $m_c$ and $H_c$ are computed per cell. \Cref{tab:confirmatory_headroom} estimates $H_c$ on one half of each cell and $\Delta_c$ on the other, so the split cannot be tuned to the same outcomes it is asked to predict. The main check is ordering: noise from disjoint estimation may shrink the gap, but the positive-margin bucket should stay above the nonpositive-margin bucket. It does, across discovery, sample-disjoint, held-out, and threshold-sweep rows. \Cref{sec:app_confirmatory_tables,sec:app_phi_within_task} give the bootstrap, sensitivity, permutation, per-baseline, and prompt-level checks.

\Cref{tab:stage_ablation} removes the leverage multiplier. Across 264 nonduplicated matched cells, NoLev and final MII signs agree on 239 cells (90.5\%) after binning changes against SnapKV as positive or nonpositive. The leverage residual, MII minus NoLev, matches the final MII sign on 162 cells (61.4\%), rescues 12 NoLev-nonpositive cells, and breaks 13 NoLev-positive cells. In this binary diagnostic, the value-consequence score has higher sign agreement with final MII than the leverage residual. Leverage changes a smaller set of boundary cases. Appendix~\ref{sec:app_value_channel_negative_control} runs the stronger no-value ablation by removing $T^u$ itself. That separation collapses. NoLev controls leverage rather than value. Appendix~\ref{sec:app_boundary_split_controls} gives the boundary-level check, and Appendix~\ref{sec:app_stage_registry} lists the supporting rows.

\begin{table}[t]
\begin{nolinenumbers}
\centering
\caption{Confirmatory reference-margin checks preserve the ordering gap under disjoint and held-out splits. Extends \Cref{tab:mii_headroom} with interval estimates for the same signed reference-margin logic.}
\label{tab:confirmatory_headroom}
\scriptsize
\setlength{\tabcolsep}{2pt}
\renewcommand{\arraystretch}{1.03}
\begin{threeparttable}
\newcommand{\cellci}[2]{\makecell[c]{#1\\{\tiny #2}}}
\begin{tabular*}{\textwidth}{@{\extracolsep{\fill}}l S[table-format=2.0] c c c c@{}}
\toprule
& & \multicolumn{2}{c}{\textbf{$m_c>\tau$}} & \multicolumn{2}{c}{\textbf{$m_c\le -\tau$}} \\
\cmidrule(lr){3-4}\cmidrule(lr){5-6}
\textbf{Analysis} & \textbf{Cells} & \textbf{Match (\%)} & \textbf{Mean \(\Delta\) (pp)} & \textbf{Match (\%)} & \textbf{Mean \(\Delta\) (pp)} \\
\midrule
Discovery, fixed $H_c$ ($\tau = 0$) & 96 & 72.6 & \cellci{+0.54}{[$+0.22$, $+0.83$]} & 67.6 & \cellci{-0.44}{[$-0.80$, $-0.10$]} \\
\midrule
Sample-disjoint, raw ($\tau = 0$) & 96 & \cellci{64.0}{[54.6, 73.1]} & \cellci{+0.32}{[$+0.06$, $+0.55$]} & \cellci{59.8}{[46.2, 72.8]} & \cellci{-0.20}{[$-0.45$, $+0.04$]} \\
Sample-disjoint, margin ($|m_c| \ge 0.5$ pp) & 96 & \cellci{68.7}{[60.5, 76.2]} & \cellci{+0.39}{[$+0.18$, $+0.60$]} & \cellci{67.0}{[56.5, 76.8]} & \cellci{-0.25}{[$-0.46$, $-0.04$]} \\
Sample-disjoint, mechanism & 96 & \cellci{69.3}{[60.6, 77.1]} & \cellci{+0.48}{[$+0.21$, $+0.75$]} & \cellci{55.6}{[44.7, 66.4]} & \cellci{-0.07}{[$-0.37$, $+0.22$]} \\
\midrule
Held-out, low-budget & 96 & \cellci{67.6}{[57.4, 76.8]} & \cellci{+0.31}{[$+0.11$, $+0.51$]} & \cellci{64.2}{[53.4, 73.9]} & \cellci{-0.18}{[$-0.34$, $-0.02$]} \\
Held-out, prompt split & 96 & \cellci{66.9}{[56.8, 75.7]} & \cellci{+0.30}{[$+0.05$, $+0.55$]} & \cellci{65.4}{[53.4, 75.9]} & \cellci{-0.17}{[$-0.40$, $-0.01$]} \\
\bottomrule
\end{tabular*}
\begin{tablenotes}[flushleft]
\footnotesize
\item Each mean cell reports the point estimate on the first line and the reported confidence interval on the second line. The favorable invariant is ordering: the predicted-positive bucket should remain above the predicted-nonpositive bucket even when the split becomes noisier.
\end{tablenotes}
\end{threeparttable}
\end{nolinenumbers}
\vspace{-0.4em}
\end{table}

\begin{table}[t]
\begin{nolinenumbers}
\centering
\caption{Removing leverage weakens the effect but keeps the signed reference-margin ordering. Companion to \Cref{tab:mii_headroom}.}
\label{tab:stage_ablation}
\scriptsize
\setlength{\tabcolsep}{3.5pt}
\renewcommand{\arraystretch}{1.03}
\begin{tabular*}{\textwidth}{@{\extracolsep{\fill}}l S[table-format=2.1] S[table-format=+1.2] S[table-format=2.1] S[table-format=+1.2]@{}}
\toprule
& \multicolumn{2}{c}{\textbf{SnapKV + NoLev}} & \multicolumn{2}{c}{\textbf{SnapKV + full replacement}} \\
\cmidrule(lr){2-3}\cmidrule(lr){4-5}
\textbf{Group} & \textbf{Positive shift (\%)} & \textbf{Mean \(\Delta\) from SnapKV (pp)} & \textbf{Positive shift (\%)} & \textbf{Mean \(\Delta\) from SnapKV (pp)} \\
\midrule
All 96 cells           & 57.3 & +0.15 & 58.3 & +0.20 \\
$H_c > 0$                & 69.4 & +0.48 & 72.6 & +0.54 \\
$H_c \le 0$             & 29.4 & -0.34 & 32.4 & -0.44 \\
low-$\phi$, $H_c > 0$    & 74.5 & +0.61 & 76.4 & +0.69 \\
high-$\phi$, $H_c > 0$   & 28.6 & -0.42 & 42.9 & -0.59 \\
low-$\phi$, $H_c \le 0$ & 30.4 & -0.41 & 34.8 & -0.49 \\
high-$\phi$, $H_c \le 0$ & 27.3 & -0.29 & 27.3 & -0.34 \\
\bottomrule
\end{tabular*}
\end{nolinenumbers}
\vspace{-0.4em}
\end{table}



\Needspace{8\baselineskip}
\parhead{Contract-drift evaluation.}
\label{sec:comparison_surfaces}
SnapKV~\citep{li2024snapkv} is the identity branch for one observation-window contract. \Cref{tab:baseline_viability} tests transfer when other selectors change access estimation, allocation, or projection. PyramidKV~\citep{cai2024pyramidkv}, CAKE~\citep{qin2025cake}, and Ada-KV~\citep{feng2026adakv} move along Stage~III by changing pyramidal layer budget, layer cascade, or head budget. H2O-debiased~\citep{zhao2025attentiondebiasing} and H2O~\citep{zhang2023ho} move along Stage~I by changing cumulative-support estimation.
For a drift surface $B$, $m_B=\fullkv{}-B$ is the contract-local reference margin and $H_B=\operatorname{sign}(m_B)$ gives the split.

\par\vspace{0.10em}
\begin{nolinenumbers}
\noindent\begin{minipage}{\textwidth}
\centering
\captionsetup{type=table,hypcap=false,skip=2pt}
\caption{Value-channel lift transfers across Stage~III drift and weakens on Stage~I surfaces in the cross-contract evaluation. Baselines are SnapKV~\citep{li2024snapkv}, PyramidKV~\citep{cai2024pyramidkv}, CAKE~\citep{qin2025cake}, Ada-KV~\citep{feng2026adakv}, and H2O~\citep{zhang2023ho}.}
\label{tab:baseline_viability}
\scriptsize
\setlength{\tabcolsep}{2pt}
\renewcommand{\arraystretch}{1.02}
\begin{tabular*}{\textwidth}{@{\extracolsep{\fill}}l S[table-format=+2.2] S[table-format=+2.2] S[table-format=2.0] S[table-format=2.0] c p{0.18\textwidth}@{}}
\toprule
& \multicolumn{2}{c}{\textbf{Surface gap (pp)}} & \multicolumn{2}{c}{\textbf{Contract-local margin counts}} & & \\
\cmidrule(lr){2-3}\cmidrule(lr){4-5}
\textbf{Surface $B$} & \textbf{$B-\mathrm{SnapKV}$} & \textbf{$B-\fullkv{}$} & \textbf{$H_B>0$} & \textbf{$H_B\le0$} & \textbf{Mean \(\Delta_B\) (pp)} & \textbf{Dominant stage} \\
\midrule
SnapKV    & +0.00  & -1.79  & 62 & 34 & +0.19 & Stage~II \\
PyramidKV & -2.50  & -4.37  & 73 & 23 & +0.84 & Stage~III, layer \\
CAKE      & -4.51  & -6.30  & 78 & 18 & +0.30 & Stage~III, cascade \\
Ada-KV    & -2.11  & -3.90  & 76 & 20 & +0.15 & Stage~III, head \\
H2O-deb.  & -7.68  & -9.47  & 83 & 13 & \multicolumn{1}{c}{--} & Stage~I, debiased \\
H2O       & -24.81 & -26.60 & 91 &  5 & \multicolumn{1}{c}{--} & Stage~I, cumulative \\
\bottomrule
\end{tabular*}
\end{minipage}
\end{nolinenumbers}
\par\vspace{0.10em}

\Cref{tab:baseline_viability} shows the contract-drift evaluation. The Stage~III adapter lifts follow the stage map, with $+0.84$\,pp on PyramidKV, $+0.30$\,pp on CAKE, and $+0.15$\,pp on Ada-KV. These rows test drift sensitivity, not global selector rank. Value-consequence information transfers when access support remains comparable, and lift decays as Stage~III absorbs or distorts that signal. \Cref{sec:app_pyramidkv_contract} isolates the remaining PyramidKV layer-budget gap.

H2O-debiased and H2O move the bottleneck upstream into Stage~I. A Stage~II replacement stacked on those selectors would receive a different support distribution, so any gain would mix value repair with a changed support law. The Stage~I correlation between debias lift and access gap is $r{=}0.88$ on the same 96 cells (\Cref{fig:stagei_predictability}, Appendix~\ref{sec:app_conditional_provenance}), in contrast to the Stage~III $r{=}0.14$ (\Cref{fig:stageiii_separability}c). This contrast supports stage separability.

\parhead{Projection and multi-target stress tests.}
\label{sec:ablation}
One Stage~III explanation is that compressed selectors fail because they leave part of the budget unused. Strict block TopK retains $k_b=\lfloor k/p\rfloor$ full blocks under budget $k$ and block size $p$, leaving a lattice residual $\varepsilon_{\mathrm{lat}}(k,p)=k-pk_b$ tokens unspent across cells (\Cref{fig:stageiii_separability}a). Token-fill is the diagnostic control: after block selection, it spends the residual on individual tokens without changing Stage~I or Stage~II rankings. If unused budget were the active Stage~III bottleneck, filling it should restore performance. Token-fill reduces the per-model mean unused-token slack from 12.0, 10.0, and 8.9 tokens to under 0.1 token (\Cref{fig:stageiii_separability}b). Yet the per-cell score change stays weakly correlated with the original slack ($r{=}{+}0.14$, mean $+0.12$\,pp, median $0.00$\,pp), and Stage~II positive and nonpositive cells interleave without separation (\Cref{fig:stageiii_separability}c). Matching the token count does not recover Stage~III performance in this control. Unused budget is not supported as the active standalone error, and the residual points to block projection and support closure.

\begin{figure}[t]
\centering
\begin{minipage}[t]{0.40\linewidth}
\centering
\includegraphics[height=0.88in]{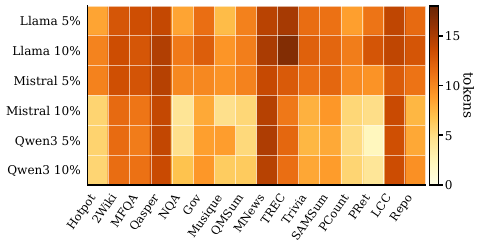}\\
{\footnotesize (a) Strict-block unused-token slack.}
\end{minipage}\hfill
\begin{minipage}[t]{0.24\linewidth}
\centering
\includegraphics[height=0.88in]{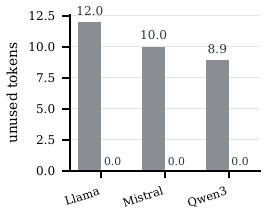}\\
{\footnotesize (b) Per-model slack before and after token-fill.}
\end{minipage}\hfill
\begin{minipage}[t]{0.32\linewidth}
\centering
\includegraphics[height=0.88in]{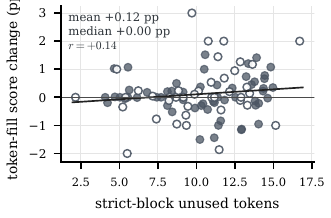}\\
{\footnotesize (c) Score change against strict-block slack.}
\end{minipage}
\caption{Stage~III separability under block projection. (a) Strict-block lattice slack across the 96-cell grid, with rows for model--budget pairs, columns for LongBench~\citep{bai2024longbench} tasks, and color showing unused tokens. (b) Per-model mean slack before and after token-fill. (c) Per-cell score change after token-fill against the original slack, where filled markers are Stage~II positive cells and open markers are nonpositive.}
\label{fig:stageiii_separability}
\end{figure}

The multi-target stress test probes support closure under branched decode queries. The staged diagnostic predicts that a single pooled proxy can lose branch-specific evidence because related fields must survive together. Appendix~\ref{sec:oos_prediction} reports the 32k NeedleBench M-RT rows and a directional RULER~\citep{hsieh2024ruler} 8k check. The former shows pooling losing branch-specific evidence even when individual tokens receive high scores; the latter tests whether logged movement concentrates on multikey retrieval. These rows are stress checks for the projection and support-closure failure mode, not clean Stage~II value-channel evidence. Contracts with allocation drift or broken Stage~I support remain scoped to \Cref{tab:baseline_viability}.

\section{Discussion and Conclusion}
\label{sec:discussion}

A KV compression score is not self-diagnosing. The same task-level change can come from access estimation, value consequence, or projection. Fixing the selector contract lets these stages be tested one at a time. Compression also need not be monotone: eviction can move the value-channel decision boundary, so a compressed cache can sometimes exceed \fullkv{}. Token-fill gives the Stage~III countercheck (\Cref{fig:stageiii_separability}). It spends the strict block lattice slack, yet the score surface barely moves ($r{=}{+}0.14$, mean $+0.12$\,pp). The remaining error sits in block projection and support closure.

Changing the observation window, layer or head allocation, or score-to-cache projection changes the stage under test. Such comparisons still matter for deployment, but they do not identify which component was repaired. The diagnostic order is to fix the contract, perturb one stage, check the predicted sign split, and assign the failure to a stage before optimizing the selector.

Reporting should follow the same logic. A final average delta can pool small wins, boundary rescues, neutral or tied cells, and losses that cancel them. Counts for neutral and tied cells, budget sweeps, and stage-local ablations show whether the movement traces to missing access support, value-channel repair, leverage at a boundary, or projection that drops coupled evidence. Guards and fallback rules should be measured under the same fixed contract.



\bibliographystyle{plainnat}

\newpage
\appendix

\section*{Appendix}
\begin{itemize}[leftmargin=1.5em,itemsep=0.12em,topsep=0.25em,parsep=0pt]
    \item Appendix~\ref{sec:app_notation}: notation and technical preliminaries.
    \item Appendix~\ref{sec:app_main_text_proofs}: proofs for the main-text formal statements.
    \item Appendix~\ref{sec:related}: related work.
    \item Appendix~\ref{sec:app_algorithm}: algorithmic details.
    \item Appendix~\ref{sec:appendix_rate_limited_proxy}: the multi-objective formal backbone.
    \item Appendices~\ref{sec:appendix_exposure_bias_theory}--\ref{sec:app_interpretability}: derivations and diagnostics for exposure bias, head aggregation, and layer-wise access alignment.
    \item Appendix~\ref{sec:oos_prediction}: multi-objective retrieval stress tests.
    \item Appendix~\ref{sec:app_conditional_provenance}: the stage registry, predictor provenance, robustness checks, and transfer tables.
    \item Appendix~\ref{sec:app_additional_controls}: additional confirmatory controls.
    \item Appendices~\ref{sec:app_limitations}--\ref{sec:app_broader_impacts}: limitations and broader impacts.
\end{itemize}

\section{Notation and Technical Preliminaries}
\label{sec:app_notation}

This appendix collects symbols reused across sections. Symbols introduced only for a local derivation are defined at their first use.

\subsection{Notation}
\label{sec:app_notation_list}

Scalars and indices use italic letters, frequently reused sets and index families such as $\mathcal K$ or $\mathcal U$ use calligraphic letters, tensors and vectors use bold symbols (e.g.\ $\mathbf{A}$, $\mathbf{V}$, $\boldsymbol{\mu}$), probability laws are written in plain italic, and operators are written in roman type. Scalar entries and scalar aggregates are not bolded.

\begin{itemize}[leftmargin=1.6em,itemsep=1pt,topsep=2pt,parsep=0pt]
\item $[n]$: the index set $\{1,\ldots,n\}$.
\item $\operatorname{TV}(\mu,\nu)$: total-variation distance between probability laws.
\item $\operatorname{TopK}(s,k)$: the deterministic set of $k$ largest-score indices.
\item $\operatorname{Agg}_{l,h}$: the fixed layer-head aggregation operator used by a selector contract.
\item $T$, $b$, $k$: prompt length, budget ratio, and retained-token budget $k=\lfloor bT\rfloor$.
\item $i$, $u$, $l$, $h$: key position, query or proxy position, layer, and query head.
\item $\mathbf{A}$, $\mathbf{A}_{l,h}$, $\mathbf{A}_{l,h}[u,i]$: prefill-attention tensor, its layer-head slice, and one scalar entry.
\item $\mathbf{V}_{l,\kappa(h),i}$, $\mathbf{o}_{l,h,u}$, $\boldsymbol{\mu}_{l,h,u}(c)$: value vector at one position, attention output vector, and block centroid vector.
\item $s_i$, $\mathcal K$: scalar score and retained cache positions.
\item $u_{\mathrm{acc}}(i)$: decode-side access support.
\item $u_{\mathrm{val}}(i)$: conditional value consequence.
\item $\eta_{\mathrm{proj}}$: projection residual after mapping scores to a retained set.
\item $\hat q_m$, $\hat d_m$, $\hat f_m$: query law, exposure correction, and pooling kernel of estimator $m$.
\item $q^\star$, $d^\star$, $f^\star$: reference counterparts in the Stage~I decomposition.
\item $\delta_{\mathrm{phase}}$, $\delta_{\mathrm{exp}}$, $\xi_{\mathrm{acc}}$: Stage~I query-law, exposure, and aggregation residuals.
\item $\Delta$, $\hat\Delta$, $\xi_{\mathrm{val}}$: true score, estimated score, and value-channel residual in Stage~II.
\item $p^\star$, $p_T$, $\bar p_{\mathcal K}$: first-decode law, last-prefill law, and retained-set-renormalized law.
\item $\phi(x)$, $m_c$, $H_c$, $\Delta_c$: support-coupling index, signed reference margin $m_c=\fullkv{}_c-\mathrm{Host}_c$, its sign $H_c=\operatorname{sign}(m_c)$, and intervention outcome for cell $c$.
\item $\mathcal B$, $c$, $p$, $k_b$: block partition, block, block size, and block budget $k_b=\lfloor k/p\rfloor$.
\item $\mathcal L_\star$, $\beta_l$: selected layer set and layer weight.
\item $\mathcal U$, $\mathcal U^{\mathrm{tail}}$, $\mathcal U^{\mathrm{anchor}}$, $r_u$: proxy bank, its tail and anchor components, and proxy weight.
\item $a_{l,h,u}(c)$, $d_{l,h,u}(c)$: block mass scalar and deletion-cost surrogate.
\item $D_m(c)$, $S(c)$, $\mathcal Z^\star$: group-wise score, pooled block score, and selected-block index set.
\item $\varepsilon_{\mathrm{lat}}$: block-lattice token slack, distinct from $\eta_{\mathrm{proj}}$.
\item $\mu_{\mathrm{dec}}$, $\tilde\eta$, $r$, $\sigma$: labeled decode law, proxy law, channel count, and routing map.
\end{itemize}

Localized symbols not listed above--for example $\gamma_{\varepsilon,l,h,u}$, $\varepsilon_a$, $\varepsilon_\mu$, $R_2$, $M$, $\mathcal U_m$, $w_m$, $\nu_m$, $w_{\max}$, $\lambda_{\mathrm{eff}}$, and $N_i^\star$--are defined at their point of use.

\section{Proofs for the Staged Diagnostic Framework}
\label{sec:app_main_text_proofs}

\subsection{Boundary-Margin Condition}
\label{sec:proof_boundary_margin}

\noindent\textit{Proof of \Cref{lem:boundary_margin}.}
Under the base score, $j$ is retained and $i$ is evicted, so $s_j\ge s_i$ up to the fixed tie break. Under the perturbed score $s'=s+\delta$, token $i$ strictly ranks above token $j$ whenever
\[
s_i+\delta_i > s_j+\delta_j,
\]
which is equivalent to $\delta_i-\delta_j>s_j-s_i$; equality is resolved by the fixed tie break. For a one-for-one boundary swap, the kept set changes from $\mathcal K_s$ to $(\mathcal K_s\setminus\{j\})\cup\{i\}$. If latent utility is locally additive for this swap, the utility difference is
\[
\sum_{r\in(\mathcal K_s\setminus\{j\})\cup\{i\}}u_r-\sum_{r\in\mathcal K_s}u_r
=u_i-u_j .
\]
Away from the equality case, the perturbation helps at that boundary when the margin is crossed and the entering token has larger latent utility than the leaving token. If the margin is not crossed, the pair does not swap; if it is crossed in the opposite utility direction, the swap is harmful.

\subsection{Ordered-Substitution Identity}
\label{sec:proof_ordered_substitution}

\noindent\textit{Proof of \Cref{prop:ordered_substitution}.}
Starting from the two estimators in \Cref{eq:access_support_estimators}, add and subtract the intermediate estimators that replace one component at a time:
\begin{align*}
\hat u_{\mathrm{acc}}(i;m)-u_{\mathrm{acc}}(i)
&=
\sum_{u\in\mathcal Q_m} \hat q_m(u)\hat d_m(i,u)\hat f_m(u,i)
-
\sum_{u\in\mathcal Q_m} q^\star(u)d^\star(i,u)f^\star(u,i)\\
&=
\sum_{u\in\mathcal Q_m} \bigl(\hat q_m(u)-q^\star(u)\bigr)d^\star(i,u)f^\star(u,i)\\
&\quad+
\sum_{u\in\mathcal Q_m} \hat q_m(u)\bigl(\hat d_m(i,u)-d^\star(i,u)\bigr)f^\star(u,i)\\
&\quad+
\sum_{u\in\mathcal Q_m} \hat q_m(u)\hat d_m(i,u)\bigl(\hat f_m(u,i)-f^\star(u,i)\bigr).
\end{align*}
The three displayed terms are exactly $\delta_{\mathrm{phase}}(i;m)$, $\delta_{\mathrm{exp}}(i;m)$, and $\xi_{\mathrm{acc}}(i;m)$ in \Cref{eq:decomp}. The identity is therefore a telescoping decomposition and does not use Taylor expansion or an independence assumption.

\subsection{Tail-K TV Bound}
\label{sec:proof_tailk_tv}

\noindent\textit{Proof of \Cref{prop:last_prefill_anchor}.}
Both inequalities are the triangle inequality for total variation distance. Applying it with the intermediate law $p_T$ gives
\[
\operatorname{TV}(p^\star,\bar p_{\mathrm{tail},K})
\le
\operatorname{TV}(p^\star,p_T)
+
\operatorname{TV}(p_T,\bar p_{\mathrm{tail},K}),
\]
which proves \Cref{eq:tailk_tv_bound}. Replacing $\bar p_{\mathrm{tail},K}$ by $\bar p_{\mathrm{all}}$ gives \Cref{eq:allprefill_tv_bound}.

\section{Related Work}
\label{sec:related}

We group prior work by the score proxy used to select, store, or read cache states.

\parhead{Token-level eviction.}
H2O~\citep{zhang2023ho} uses cumulative attention; StreamingLLM~\citep{xiao2024efficient} retains sinks plus a recency window; SnapKV~\citep{li2024snapkv} uses an observation window of the last $W$ queries per layer; PyramidKV~\citep{cai2024pyramidkv} uses pyramidal layer-budget allocation; Ada-KV~\citep{feng2026adakv} adapts per-head budgets; and R-KV~\citep{cai2025r} adds redundancy-aware selection for reasoning models.
These methods still depend on attention-derived scores, so they can fail when the observed prefill signal differs from the decode-time utility. Our analysis explains why low-budget results separate and why some methods become similar once the largest mismatch is repaired.
A recent survey~\citep{su2026attention} reviews attention-sink utilization, interpretation, and mitigation, including the practice of retaining high-cumulative-attention anchors during compression. The same survey notes that a unified theoretical framework remains open.
Online reactive policies such as TOVA~\citep{oren2024transformers} are out of contract for the confirmatory Stage~II test because they change both the observation law and the generation-time cache dynamics.

\parhead{Bias-aware and decode-aligned eviction.}
AhaKV~\citep{gu2025ahakv} formalizes the position-dependent bias in cumulative attention; Attention Debiasing~\citep{zhao2025attentiondebiasing} applies the same $\div(T{-}i)$ correction for VLMs.
We place exposure bias inside the Stage~I access-support estimator alongside query-law mismatch (Phase Dilution), which explains failure patterns a count fix alone cannot, including the reverse-depth behavior after isolated count repair.
DapQ~\citep{tian2026matters} studies the same decode-alignment axis by reporting higher scores for position-aware pseudo-queries than for cumulative scoring; Expected Attention~\citep{devoto2025expected} estimates attention under a future-query distribution; ForesightKV~\citep{dong2026foresightkv} and LookaheadKV~\citep{ahn2026lookaheadkv} use learned or draft-model surrogates for future access.
The Stage~II clean grid uses the zero-extra-forward-pass operating point on this axis; a matched-compute control confirms that pseudo-query scoring yields only a small grand-mean gain at nontrivial additional prefill cost.
ReST-KV~\citep{an2026restkv} takes a different approach, replacing raw attention proxies with an output-aware surrogate built from layer-wise reconstruction. Its objective is reconstruction; our diagnosis addresses score misranking at the eviction boundary.
Early attention analysis showed that heads carry heterogeneous linguistic, positional, and syntactic patterns rather than a single uniform importance signal~\citep{clark2019does}.
In the VLA domain, TIES~\citep{liu2026beyond} observes that high-attention tokens are not always important; reasoning-token analysis~\citep{singh2026llms} and direction-aware RLVR update analysis~\citep{huangbeyond} reach related conclusions about magnitude-only indicators.
The attention tensor contains decode-utility information, but the way it is aggregated determines whether that information survives into the final score; head-functional differentiation, identified by RazorAttention~\citep{tang2025razorattention} as a small group of retrieval-relevant heads, is the head-side counterpart of the same observation.
CriticalKV~\citep{feng2025identify}, VATP~\citep{guo2024attention}, DefensiveKV~\citep{feng2026defensivekv}, and related value-aware eviction variants modify scoring or aggregation inside a fixed selection contract. The staged diagnostic treats these changes as value-consequence interventions and asks when the surrounding contract leaves a recoverable signed reference margin. A value-aware score is a clean Stage~II intervention only when the access estimator and projection rule are held fixed.

\parhead{Soft query proxies and future-aware scoring.}
A growing line of work replaces hard tail-only observation windows with softer or more forward-looking query representations.
LAQ~\citep{wang2025lookahead} generates low-cost pseudo lookahead queries and uses them as the observation window for importance estimation.
LookaheadKV~\citep{ahn2026lookaheadkv} adds parameter-efficient modules that predict future-response importance without explicit draft generation.
KVzip~\citep{kim2025kvzip} takes a query-agnostic route, compressing the cache so that the original context is reconstructible under arbitrary downstream queries.
These works address the gap between prefill-derived importance scores and true decode queries through lookahead or query-agnostic corrections.
Our focus is when that gap changes the outcome. Under a fixed observation-window contract, the signed reference margin and support coupling predict whether a value-channel correction helps or harms.
On the token side, GraphKV~\citep{li-etal-2025-graphkv} adds redundancy-aware reranking through similarity-graph diffusion, addressing the coverage failure seen in multi-objective retrieval.
The fixed-size prefill-visible proxy bank used in the Stage~II clean test sits at the zero-extra-forward-pass end of this spectrum. Higher-cost proxy families target the same decode-mismatch axis with broader query support, especially for long decode trajectories and multi-objective prompts. \Cref{sec:app_factorized_sketch} states the grouped proxy objective used for the multi-target appendix under the same decode-proximal principle.

\parhead{Learned compressed-summary attention.}
Recent natively trained sparse-attention architectures replace raw token-level history with learned compressed summaries and hardware-aligned block reads. NSA organizes history into compressed coarse-grained tokens, selected fine-grained blocks, and a sliding window, while DeepSeek-V4 combines Compressed Sparse Attention and Heavily Compressed Attention to support million-token contexts through learned compressed KV entries~\citep{yuan2025native,deepseekai2026deepseekv4}. State-space sequence models such as Mamba~\citep{gu2024mamba} avoid dense-attention KV growth through a different recurrent-state substrate. These methods change the model architecture and the cache state itself. Our setting is different: we study non-reconstructing eviction for pretrained dense-attention models after training. Blocks are selection and projection units; the retained cache contains original KV states, and block centroids are used only to score deletion consequence.

\parhead{Allocation, compression, and orthogonal directions.}
FastGen~\citep{ge2023model} uses head-specific retention; CAKE~\citep{qin2025cake}, ARKV~\citep{lei2026arkv}, and adaptive layer selection~\citep{taniguchi2026adaptive} allocate cache across layers, pruning stages, or precisions given an importance score.
The diagnostic framework operates at the prior question of which score to trust.
Storage-side and token-adaptive compression~\citep{hooper2024kvquant,liu2024kivi,nawrot2024dynamic,lu2026one}, architecture-level sparse attention~\citep{bai2026indexcache}, and residual-attention architectures~\citep{team2026attention} act on different axes; applying the score diagnostic to them requires a separate fixed-contract interface rather than a direct head-to-head eviction comparison.
FleetOpt~\citep{chen2026fleetopt} treats compression as part of fleet provisioning and routing, with selector scoring as one component.
Recent failure-oriented evaluations and system studies~\citep{chen2025pitfalls,sood2026more,bocharnikov2026offloading} document when compression or offloading breaks; our contribution is to assign those failures to access estimation, value ranking, or projection under a fixed contract.

\parhead{Boundary for reconstructing and non-eviction compression.}
We study \emph{non-reconstructing} eviction, where methods drop tokens without compensating their value. Numerical-recovery quantization~\citep{liu2024kivi,hooper2024kvquant} keeps every token but lowers precision; low-rank latent reconstruction, as in Palu~\citep{chang2025palu}, regenerates key-value states at decode time from a compressed cache; importance-aware mixed-precision quantization, as in MiKV~\citep{yang2024no}, keeps cache entries while assigning lower bit-widths to less important positions; compensation-token methods, as in RazorAttention~\citep{tang2025razorattention}, inject summary tokens for dropped content; and context-reconstruction-guided compression, as in KVzip~\citep{kim2025kvzip}, optimizes for reproducing the original context. These methods operate on different axes and fall outside the head-to-head eviction comparison; the staged diagnostic applies to the eviction step itself, separate from the storage or reconstruction layer.

\section{Algorithm Details}
\label{sec:app_algorithm}

This appendix specifies the default block-level selector and the controlled extensions used in appendix ablations.
It gives the reference pseudocode, the support-only diagnostic channel, the extension hierarchy, and the identity check that recovers the canonical recipe.
Each optional component is separated from the default method used in the main experiments.

\Cref{alg:pap} gives the reference implementation of the canonical block-level default.
Appendix ablations modify this recipe in three disabled-by-default ways: grouped proxy banks for explicit multi-target prompts, soft-robust aggregation across groups, and a small reserve for heavy-tail outlier blocks.

\begin{algorithm}[H]
\caption{Reference implementation of the staged scoring pipeline}
\label{alg:pap}
\small
\begin{algorithmic}[1]
\STATE Partition prompt into blocks $\mathcal{B}=\{c_1,\dots,c_{N_b}\}$ of size $p$; $k \gets \lfloor bT \rfloor$; $k_b \gets \lfloor k/p \rfloor$; $S(c)\gets 0$ for all $c$
\STATE \textbf{Module~I: recover decode-side access support $u_{\mathrm{acc}}$}
\STATE Select late or plateau layers $\mathcal{L}_\star$ with normalized weights $\{\beta_l\}$ \hfill $\triangleright$ \textit{reduces $\xi_{\mathrm{acc}}$}
\STATE $\mathcal{U}^{\mathrm{tail}} \gets \{T{-}w{+}1,\dots,T\}$; \ $\mathcal{U}^{\mathrm{anchor}} \gets \mathcal{C}_{\mathrm{anchor}}$; \ $\mathcal{U} \gets \mathcal{U}^{\mathrm{tail}} \cup \mathcal{U}^{\mathrm{anchor}}$ \hfill $\triangleright$ \textit{fixes $\delta_{\mathrm{phase}}$}
\STATE $r_u \propto \exp(-(T{-}u)/\tau_q)$ on tail, $\propto 1$ on anchor; normalize \hfill $\triangleright$ \textit{cancels $\delta_{\mathrm{exp}}$}
\STATE \textbf{Module~II: reweight by conditional output consequence $u_{\mathrm{val}}$}
\FOR{each $l \in \mathcal{L}_\star$, $u \in \mathcal{U}$, and query head $h$}
\STATE $\mathbf{o}_{l,h,u} \gets \sum_{i=1}^{T} \mathbf{A}_{l,h}[u,i]\,\mathbf{V}_{l,\kappa(h),i}$
\FOR{each block $c \in \mathcal{B}$}
\STATE $a \gets \sum_{i\in c} \mathbf{A}_{l,h}[u,i]$, \ $\boldsymbol{\mu} \gets \textstyle\sum_{i\in c} \mathbf{A}_{l,h}[u,i]\,\mathbf{V}_{l,\kappa(h),i}/(a{+}\varepsilon)$
\STATE $S(c) \mathrel{+}= r_u\,\beta_l\,\bigl(a/(1{-}a{+}\varepsilon)\bigr)^2\,\lVert\boldsymbol{\mu} - \mathbf{o}_{l,h,u}\rVert_2^2$
\ENDFOR
\ENDFOR
\STATE \textbf{Module~III: block-constant posterior approximation} \hfill $\triangleright$ \textit{implemented by strict block TopK}
\STATE \textbf{Output:} $\mathcal{K} \gets \bigcup\bigl\{c : S(c) \in \mathrm{TopK}_{k_b}(\{S(c')\}_{c'\in\mathcal{B}})\bigr\}$
\end{algorithmic}
\end{algorithm}

\textbf{Support-only diagnostic score.}\nobreak\hspace{0.25em}Aggregating attention mass alone gives the support-only diagnostic
\[
\operatorname{Acc}(c)=
\sum_{u\in \mathcal{U}} r_u
\sum_{l\in\mathcal{L}_\star}\beta_l
\sum_h a_{l,h,u}(c),
\]
which answers only whether block $c$ lies in the predicted decode-side support region. It is separate from Module~II's output-distortion surrogate; we expose it in the runner as an optional channel for Module~I-only ablations.

The extension hierarchy is separate from the default selector.
Grouped proxies are introduced only when the prompt explicitly contains multiple competing targets.
Soft-robust aggregation replaces the pooled weighted sum with a log-sum-exp score over groups, and the reserve branch keeps $r$ blocks with the largest groupwise distortion.
\textbf{Identity check.}\nobreak\hspace{0.25em}When $M{=}1$, soft-robust aggregation is disabled, and $r{=}0$, these choices reduce exactly to the pooled block-aligned default in Algorithm~\ref{alg:pap}.
In that case, $D_1(c)$ is the only group score, $S(c)=D_1(c)$, $\mathcal{Z}_{\mathrm{res}}=\emptyset$, and the selected block set is the standard top-$k_b$ set.
Soft-robust aggregation and the reserve branch are appendix extensions, not part of the main default.
The multi-target stress case is reported in \Cref{sec:oos_prediction}.

\subsection{\texorpdfstring{Grouped Proxy Objective under the $r$-Channel Bound}{Grouped Proxy Objective under the r-Channel Bound}}
\label{sec:app_factorized_sketch}
\label{app:pap_factorized}

Under the mode-separation assumption of \Cref{thm:r_channel_tv}, every $r{=}1$ proxy, including the pooled default, incurs a budget-independent lower bound $1-\varepsilon-w_{\max}$ on labeled multi-target decoding, with the equal-weight case reducing to $1-\varepsilon-1/n$.
The same prefill-visible proxy bank admits a grouped formulation by assigning proxies to target slots.
Let $\mathcal{U}=\mathcal{U}^{\mathrm{tail}}\cup \mathcal{U}^{\mathrm{anchor}}=\bigsqcup_{m=1}^{M} \mathcal{U}_m$ be a prefill-visible proxy bank with weights defined as in \Cref{eq:pap_principle}, and let
\[
D_m(c) \;:=\; \sum_{u \in \mathcal{U}_m} r_u \sum_{l\in\mathcal{L}_\star}\beta_l\sum_h d_{l,h,u}(c),
\]
where $d_{l,h,u}(c)$ is the block-level output-distortion surrogate used by the canonical block-level default. The corresponding soft-robust pooled score is
\[
S_{\mathrm{rob}}(c)=
\tau_g \log \sum_{m=1}^{M}\exp\!\left(\frac{D_m(c)}{\tau_g}\right),
\]
optionally combined with a tiny outlier reserve.
With $M{=}1$, this reduces to the canonical pooled setting; with $M{>}1$, it gives the grouped multi-target objective used to interpret the factorized retrieval comparison under the same prefill-visible, output-aware principle.
The grouped objective changes the channel structure only when the prompt itself exposes multiple target slots.
It is not used to tune the main LongBench results; it is the algorithmic counterpart of the $r$-channel theory in \Cref{sec:appendix_rate_limited_proxy}.

\subsection{Scope and Limitations}
\label{sec:app_algorithm_scope}

The reference selector and optional extensions below define the controlled ablations.
The default method is the pooled block-aligned recipe; grouped proxies, soft-robust aggregation, and the reserve branch are activated only for the appendix stress tests that require them.
The pseudocode assumes access to attention and value tensors during prefill and does not cover serving stacks that hide or quantize those tensors before selection.

\section{Formal Backbone for Multi-Objective Phase Dilution}
\label{sec:appendix_rate_limited_proxy}

This appendix isolates the multi-objective part of Phase Dilution as a rate-limited proxy problem.
The main object is the labeled decode law, not the unlabeled marginal.
The lower bound is therefore about target-slot collapse, not about finite-sample failure of a particular estimator.
If one drops the slot label and keeps only the unlabeled mixture $\sum_m w_m \nu_m$, a pooled proxy can match that marginal exactly, so no non-trivial TV floor independent of the target mixture survives on the unlabeled object.
The appendix defines the labeled law and the $r$-channel proxy class, then proves the lower bound and records the pooled and slot-factorized consequences used by the retrieval discussion.
All formal statements below use the fixed selector contract in \Cref{assump:fixed_contract}; the variable is the proxy channel count.

\parhead{Candidate factorized objective.}
Output-aware pruning~\citep{an2026restkv,goel2025caote} suggests a labeled multi-target variant. It keeps Module~I's support recovery and Module~II's renormalized block distortion fixed, then replaces the pooled block score by a labeled objective. The canonical pooled score is the $M{=}1$ operating point; factorized or soft-robust retention applies when $M{>}1$ and the decode law fragments across competing targets.

Let $J \in [n] := \{1,\ldots,n\}$ denote the target slot or query mode, with $\Pr(J=m) = w_m$ and $\sum_{m=1}^n w_m = 1$.
Let $\nu_m \in \mathcal{P}(\mathcal{Q})$ denote the decode-query law conditional on slot $m$.
The labeled decode law keeps the slot identity:
\begin{equation}
\mu_{\mathrm{dec}} \;:=\; \sum_{m=1}^n w_m \, \delta_m \otimes \nu_m \;\in\; \mathcal{P}([n] \times \mathcal{Q}).
\end{equation}

For $r \le n$, an $r$-channel proxy is any law of the form
\begin{equation}
\tilde{\eta} \;:=\; \sum_{m=1}^n w_m \, \delta_m \otimes \eta_{\sigma(m)},
\label{eq:r_channel_proxy}
\end{equation}
where $\sigma : [n] \to [r]$ is a routing map and $\eta_1,\ldots,\eta_r \in \mathcal{P}(\mathcal{Q})$ are shared proxy channels.
\eqref{eq:r_channel_proxy} is deliberately generous.
It already gives the proxy access to the slot identity through $\sigma(m)$.
Pooled retention is the special case $r=1$.
Slot-factorized retention is the special case $r=n$ with $\sigma(m)=m$.

The labeled decode law is $\varepsilon$-separated when there exist pairwise disjoint measurable basins $\mathcal R_1,\ldots,\mathcal R_n$ in $\mathcal{Q}$ such that
\begin{equation}
\nu_m(\mathcal R_m) \;\ge\; 1-\varepsilon, \qquad m=1,\ldots,n.
\label{eq:basin_sep}
\end{equation}
Under QK-normalized query embeddings, one concrete choice is to work on the unit sphere with geodesic distance and let $c_m$ be a prototype direction for mode $m$.
If the minimum pairwise angle is $\Delta_{\min} := \min_{m \neq m'} d(c_m, c_{m'})$, then any radius $\rho < \Delta_{\min}/2$ yields pairwise disjoint metric balls $\mathcal R_m := \{q : d(q, c_m) \le \rho\}$.
The separation assumption then reduces to requiring that each $\nu_m$ places at least $1-\varepsilon$ of its mass in its own ball.

\begin{theorem}[$r$-channel TV lower bound]
\label{thm:r_channel_tv}
Assume \eqref{eq:basin_sep} and sort the mode weights as $w_{(1)} \ge \cdots \ge w_{(n)}$.
Then every $r$-channel proxy $\tilde{\eta}$ satisfies
\begin{equation}
\operatorname{TV}(\mu_{\mathrm{dec}}, \tilde{\eta})
\;\ge\;
1 - \varepsilon - \sum_{j=1}^r w_{(j)}.
\label{eq:r_channel_tv}
\end{equation}
\end{theorem}

\begin{proof}[Proof of \Cref{thm:r_channel_tv}]
Let
\begin{equation}
A \;:=\; \{(m,q) \in [n] \times \mathcal{Q} : q \in \mathcal R_m\}.
\end{equation}
By \eqref{eq:basin_sep},
\begin{equation}
\mu_{\mathrm{dec}}(A)
=
\sum_{m=1}^n w_m \nu_m(\mathcal R_m)
\ge
1-\varepsilon.
\end{equation}
For $\tilde{\eta}$ in \eqref{eq:r_channel_proxy},
\begin{equation}
\tilde{\eta}(A)
=
\sum_{m=1}^n w_m \eta_{\sigma(m)}(\mathcal R_m).
\end{equation}
Write $G_j := \sigma^{-1}(j)$.
Since the basins are pairwise disjoint,
\begin{equation}
\sum_{m \in G_j} \eta_j(\mathcal R_m) \le 1.
\end{equation}
Therefore
\begin{equation}
\sum_{m \in G_j} w_m \eta_j(\mathcal R_m)
\le
\Bigl(\max_{m \in G_j} w_m\Bigr) \sum_{m \in G_j} \eta_j(\mathcal R_m)
\le
\max_{m \in G_j} w_m,
\end{equation}
and summing over channels gives
\begin{equation}
\tilde{\eta}(A)
\le
\sum_{j=1}^r \max_{m \in G_j} w_m.
\end{equation}
The largest possible right-hand side is attained by assigning the $r$ largest weights to distinct channels, so
\begin{equation}
\tilde{\eta}(A)
\le
\sum_{j=1}^r w_{(j)}.
\end{equation}
Finally,
\begin{equation}
\operatorname{TV}(\mu_{\mathrm{dec}}, \tilde{\eta})
\ge
\mu_{\mathrm{dec}}(A) - \tilde{\eta}(A)
\ge
1 - \varepsilon - \sum_{j=1}^r w_{(j)}.
\end{equation}
\end{proof}

Two consequences are used in the multi-target retrieval discussion. For pooled retention, $r=1$, so every pooled proxy satisfies
\begin{equation}
\operatorname{TV}(\mu_{\mathrm{dec}}, \tilde{\eta}_{\mathrm{pool}})
\;\ge\;
1 - \varepsilon - w_{\max}.
\label{eq:pooled_floor}
\end{equation}
In the equal-weight case $w_m = 1/n$, this becomes
\begin{equation}
\operatorname{TV}(\mu_{\mathrm{dec}}, \tilde{\eta}_{\mathrm{pool}})
\;\ge\;
1 - \varepsilon - \frac{1}{n}.
\end{equation}
This is \Cref{thm:r_channel_tv} with $r=1$, where $\sum_{j=1}^r w_{(j)}=w_{\max}$; the equal-weight display uses $w_{\max}=1/n$.
The bound already allows the pooled proxy to know the slot identity through the routing map $\sigma$.
Even under that favorable relaxation, a single shared channel cannot avoid label collapse.
The lower bound is an irreducible approximation floor of the one-channel class.

In contrast, let
\begin{equation}
\tilde{\eta}_{\mathrm{fact},K}
\;:=\;
\sum_{m=1}^n w_m \, \delta_m \otimes \eta_{m,K}
\end{equation}
be a slot-factorized proxy with one channel per slot.
If $\operatorname{TV}(\eta_{m,K}, \nu_m) \to 0$ for every $m$, then disjointness of the slot labels gives
\begin{equation}
\operatorname{TV}(\mu_{\mathrm{dec}}, \tilde{\eta}_{\mathrm{fact},K})
=
\sum_{m=1}^n w_m \, \operatorname{TV}(\nu_m, \eta_{m,K}),
\label{eq:factorized_consistency}
\end{equation}
so $\operatorname{TV}(\mu_{\mathrm{dec}}, \tilde{\eta}_{\mathrm{fact},K}) \to 0$ by dominated convergence.

For an end-task functional $\Phi : \mathcal{P}([n] \times \mathcal{Q}) \to \mathbb{R}$ satisfying
\begin{equation}
|\Phi(\mu) - \Phi(\nu)| \le L \, \operatorname{TV}(\mu,\nu)
\label{eq:lipschitz_task}
\end{equation}
for all $\mu,\nu$, an empirical $r$-channel proxy $\tilde{\eta}_{K,r}$ and its population counterpart $\tilde{\eta}^{\star}_{r}$ satisfy the triangle-inequality bridge
\begin{equation}
|\Phi(\mu_{\mathrm{dec}}) - \Phi(\tilde{\eta}_{K,r})|
\le
L \, \operatorname{TV}(\mu_{\mathrm{dec}}, \tilde{\eta}_{K,r})
\le
L \Bigl[
\operatorname{TV}(\mu_{\mathrm{dec}}, \tilde{\eta}^{\star}_{r})
+ \operatorname{TV}(\tilde{\eta}^{\star}_{r}, \tilde{\eta}_{K,r})
\Bigr].
\label{eq:lipschitz_bridge}
\end{equation}
If
\begin{equation}
\operatorname{TV}(\tilde{\eta}^{\star}_{r}, \tilde{\eta}_{K,r})
\le
\Delta_{\mathrm{cont}}(K) + \Delta_{\mathrm{est}}(K),
\end{equation}
then this becomes
\begin{equation}
|\Phi(\mu_{\mathrm{dec}}) - \Phi(\tilde{\eta}_{K,r})|
\le
L \Bigl[
\operatorname{TV}(\mu_{\mathrm{dec}}, \tilde{\eta}^{\star}_{r})
+ \Delta_{\mathrm{cont}}(K)
+ \Delta_{\mathrm{est}}(K)
\Bigr].
\end{equation}
Under standard LLN or mixing assumptions on the tail sample, $\Delta_{\mathrm{est}}(K) = O_{\mathbb{P}}(K^{-1/2})$.

\Cref{thm:r_channel_tv} gives a formal model for one branching bottleneck seen empirically in the 32k multi-objective retrieval table.
Pooled selection is a one-channel proxy class and retains the floor in \eqref{eq:pooled_floor} when the labeled-law assumptions hold.
Factorized selection removes that label-collapse term by allocating one channel per target slot, as shown in \eqref{eq:factorized_consistency}.
The theorem does not by itself set score ceilings or identify the optimal proxy bank.
Increasing $K$ can only reduce contamination and estimation error around that floor; it cannot remove the floor itself.
If each mode-conditioned channel has a single-mode saturation budget $K_0 = O(1)$, then the useful total budget of a factorized proxy scales as $nK_0$, whereas the pooled proxy saturates once the non-floor terms become small.

\subsection{\texorpdfstring{Tail-$K$ Consistency via Ces\`aro Sufficient Condition}{Tail-K Consistency via Cesaro Sufficient Condition}}
\label{sec:appendix_cesaro}

Let $\rho_j := \Pr(q_{T-j} \sim \mu_P)$ denote the decode-proximity rate of the $j$-th tail query, and let $\bar c_K := \frac{1}{K}\sum_{j=0}^{K-1}(1-\rho_j)$ denote the tail contamination. Under a characteristic kernel and a standard LLN/mixing condition on the tail sample, vanishing tail contamination as $K$ grows is sufficient for the tail-$K$ empirical measure $\widehat\mu_{\mathrm{tail},K}$ to converge in MMD to $\mu_P$. The converse is not used here. Persistent contamination can be MMD-invisible when its RKHS mean embedding matches or cancels against that of $\mu_P$. Under the stronger assumption that $\rho_j$ is monotone in $j$ on the tail, the Ces\`aro condition is equivalent to tail proximity rising to one as $j$ grows. The weaker Ces\`aro form is what actually underlies the diagnostic prediction in \Cref{sec:phase_dilution}, and the monotone version is a sufficient special case but is not required.

\subsection{QK-norm as Temperature Stabilization under Mismatch}
\label{sec:appendix_qknorm_tradeoff}

Let $\beta$ denote the effective post-normalization attention sharpness, and let $\beta_q$ denote the pre-normalization query-wise sharpness induced by norm fluctuations. The key role of QK-norm is to collapse the extra heteroscedasticity coming from $\beta_q$; it does not imply a monotone decrease of unconditional logit variance across all query laws. For a tail-$K$ proxy with query-distribution mismatch $D_K := W_1(\mu_{\mathrm{tail},K}, \mu_P)$, the softmax Jacobian has norm bounded by $\beta L_0$ for some $L_0 > 0$, giving
\[
\bigl\|\,S_\beta(\mu_{\mathrm{tail},K}) - S_\beta(\mu_P)\bigr\|_1 \;\le\; \beta L_0\, D_K,
\qquad
\mathrm{MSE}_K(\beta) \;\lesssim\; \beta^2 L_0^2 D_K^2 \;+\; \tfrac{V_{\mathrm{within}}(\beta) + V_{\mathrm{temp}}}{K}.
\]
Here $V_{\mathrm{temp}}$ is the query-temperature component driven by variability in $\beta_q$, and this term is precisely what QK-norm suppresses. The residual term $V_{\mathrm{within}}(\beta)$ captures within-query sampling variance and has no monotone law in $\beta$ across all query laws. The practical implication is conditional. QK-norm helps when $D_K$ is low because it removes query-wise temperature noise. Under severe mismatch, the sharper transfer can also magnify the bias term $\beta L_0 D_K$. This gives one possible mechanism for the observed Qwen3 pattern. Full-retention single-key performance can coexist with weakness on the two-key variant when the one-channel floor in \eqref{eq:pooled_floor} starts to bind.

\subsection{RoPE Locality Surrogate}
\label{sec:appendix_rope_yarn}

RoPE has a finite frequency-mixture structure. Writing
\[
K(\Delta) \;=\; \sum_{j=1}^{d_h/2} a_j \cos(\omega_j \Delta),
\qquad \omega_j = \theta^{-2(j-1)/d_h},
\]
with fixed nonnegative weights and $K(0)>0$, Taylor expansion at $\Delta=0$ gives
\[
\frac{K(\Delta)}{K(0)} \;=\; 1-\tfrac{1}{2}\lambda_{\mathrm{eff}}^2\,\Delta^2+O(\Delta^4),
\qquad
\lambda_{\mathrm{eff}}^2 := \frac{\sum_j a_j \omega_j^2}{\sum_j a_j},
\qquad
\partial \lambda_{\mathrm{eff}}^2 / \partial \theta \le 0,
\]
with strict inequality when positive weight falls on at least one $\theta$-dependent frequency. Matching this local quadratic to a Gaussian gives a small-$\Delta$ approximation, not a global envelope for finite RoPE frequency sums. Larger $\theta$ weakens this local-curvature proxy under that same weight condition. The exponential kernel in \Cref{sec:appendix_exposure_bias_theory} is used as a tractable one-parameter surrogate; the conclusions about tail-high residual curvature require a chosen non-increasing surrogate, not a global RoPE envelope.

\subsection{Scope and Limitations}
\label{sec:appendix_rate_limited_proxy_scope}

This appendix analyzes a labeled multi-target proxy class under explicit mode separation.
The bound applies to target-slot collapse in the labeled law and does not claim a floor for all unlabeled marginals.
The result explains why pooled multi-target retrieval can saturate early and why factorized retention is the relevant comparison when target slots are available in the prompt.

\section{Exposure Bias Theoretical Derivations}
\label{sec:appendix_exposure_bias_theory}

This appendix gives the derivations behind \Cref{sec:exposure_bias}: the causal-softmax decomposition, special cases, a tractable kernel surrogate, the MSE-optimal observation count, the correlated-noise extension, and the reverse-depth interpretation. Symbols follow \Cref{sec:exposure_bias}, namely $C_i$, $N_i$, $\mu(\Delta)$, $S_i$, and $\widehat{u}_i$.
Each subsection isolates one assumption change and states the implication for count debiasing.
The derivations identify which residuals the experiments should measure separately.

\Cref{fig:exposure_tail_residual} separates the residual mean from the variance after count debiasing. The mean correction bends mass toward the prompt tail, while the variance penalty appears mainly in the same tail region.

\begin{figure}[H]
	\centering
	\subfloat[Post-debias mean.]{
    \label{fig:exposure_tail_residual_a}%
	\begin{minipage}[h]{0.43\textwidth}
	\centering
    \includegraphics[width=\textwidth]{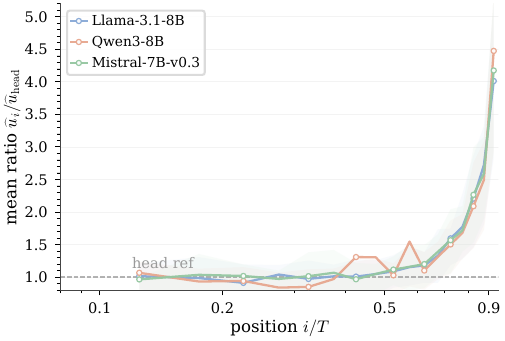}
    \end{minipage}
    }
    \hfill
    \subfloat[Post-debias variance.]{
    \label{fig:exposure_tail_residual_b}%
	\begin{minipage}[h]{0.43\textwidth}
	\centering
    \includegraphics[width=\textwidth]{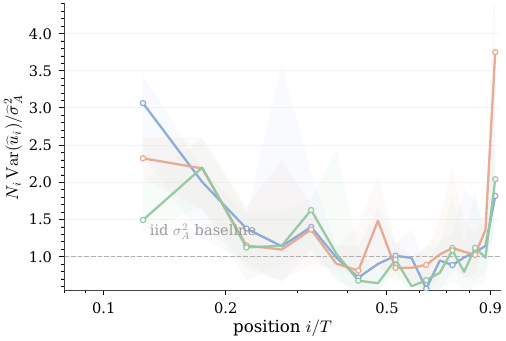}
    \end{minipage}
    }
    \caption{Exposure-debiasing residuals.}
\label{fig:exposure_tail_residual}
\end{figure}

\subsection{General Decomposition}

Absorbing layer and head sums into $A_{t,i} := \sum_{L,h} A_{L,h}(q_t, k_i)$ and writing $\mu(\Delta) := \expect[A_{t,t-\Delta}]$ and $S_i := \sum_{t=i}^T \bigl(\expect[A_{t,i} \mid \text{content}] - \mu(t-i)\bigr)$ gives
\[
\expect[C_i] \;=\; \sum_{\Delta=0}^{N_i-1} \mu(\Delta) \;+\; S_i.
\]
The main-text form $\expect[C_i] = \beta_{\mathrm{pos}} N_i + S_i$ is the special case $\mu(\Delta) \equiv \beta_{\mathrm{pos}}$. Under a decaying kernel the correct count factor is $N_i^{\mathrm{eff}} := \beta_{\mathrm{pos}}^{-1} \sum_{\Delta=0}^{N_i-1}\mu(\Delta)$.

\subsection{Exchangeable Softmax (Uniform Baseline)}

If the logit vector $(Z_{t,1},\ldots,Z_{t,t})$ is exchangeable, symmetry plus the simplex constraint $\sum_i A_{t,i} = 1$ gives $\expect[A_{t,i} \mid t] = 1/t$ for all $i \le t$. Hence
\[
\expect[C_i] \;=\; \sum_{t=i}^T \frac{1}{t} \;=\; H_T - H_{i-1},
\]
a harmonic sum of order $\Theta(\log T)$, not $\Theta(N_i)$. The linear approximation $\beta_{\mathrm{pos}} N_i$ is valid only within a thin tail window $[T-w+1, T]$ with $w = o(T)$, where $1/t = 1/T + O(w/T^2)$.

\subsection{RoPE Kernel Structure}

If the query and key distributions are rotation-invariant, $\mathbf{q}^\top \mathbf{R}_\Delta \mathbf{k} \stackrel{d}{=} \mathbf{q}^\top \mathbf{k}$, so RoPE alone induces no decay: $\mu(\Delta) \equiv \text{const}$. Non-trivial $\mu(\Delta)$ arises from learned anisotropy of the query-key spectra. Writing the head-averaged relative-position kernel as a Fourier transform $\mu(\Delta) = \int e^{i\omega\Delta}\,d\nu(\omega)$ for a finite measure $\nu$ with $L^1$ density, Riemann--Lebesgue gives convergence of $\mu(\Delta)$ to zero as $\Delta$ grows.

\Cref{sec:appendix_rope_yarn} gives a small-$\Delta$ Gaussian surrogate matched to the finite RoPE frequency sum, not a global RoPE envelope; we use the exponential form as a tractable one-parameter surrogate that retains monotone decay. The statements on residual curvature below require only that the chosen surrogate $\mu$ is non-increasing. For the exponential kernel $\mu(\Delta) = \beta_{\mathrm{pos}}\,e^{-\lambda\Delta}$,
\[
N_i^{\mathrm{eff}} \;=\; \frac{1 - e^{-\lambda N_i}}{1 - e^{-\lambda}},
\qquad
r(N) \;:=\; \frac{N_i^{\mathrm{eff}}}{N_i} \;=\; \frac{1 - e^{-\lambda N}}{(1-e^{-\lambda})\, N}.
\]
Differentiating, $r'(N) = \frac{e^{-\lambda N}(1 + \lambda N) - 1}{(1 - e^{-\lambda}) N^2} < 0$ for all $N \ge 1$ (using $e^x > 1 + x$). Naive division by $N_i$ therefore over-corrects head and middle positions and leaves a tail-high deterministic residual $\beta_{\mathrm{pos}} r(N_i)$.

\subsection{MSE-Optimal Observation Count}

Model per-query noise as iid with variance $\sigma_A^2$. Then $\widehat{u}_i = S_i/N_i + \bar{\epsilon}_i$ with $\bar{\epsilon}_i \sim \mathcal{N}(0, \sigma_A^2/N_i)$, yielding
\[
\mathrm{MSE}_i(N) \;=\; \Bigl(1 - \tfrac{1}{N}\Bigr)^{\!2} S_i^2 \;+\; \frac{\sigma_A^2}{N}.
\]
Expanding and differentiating,
\[
\frac{d\,\mathrm{MSE}_i}{dN} \;=\; \frac{(2 S_i^2 - \sigma_A^2)\,N - 2 S_i^2}{N^3},
\]
with closed-form minimizer
\[
N_i^\star \;=\; \begin{cases} \dfrac{2 S_i^2}{2 S_i^2 - \sigma_A^2} & \text{if } 2 S_i^2 > \sigma_A^2,\\[6pt] +\infty & \text{otherwise.}\end{cases}
\]
For strong-signal tokens ($S_i^2 \gg \sigma_A^2$), the optimum approaches $N_i^\star=1$, so they are best estimated at small $N_i$, at the tail. Weak-signal tokens favor $N_i^\star = \infty$ and are harmed by debiasing at the tail.

\subsection{Correlated-Noise Extension}

Allowing $\mathrm{Corr}(\epsilon_{i,r}, \epsilon_{i,r+k}) = \rho_{i,k}$,
\[
\mathrm{Var}(\widehat{u}_i) \;=\; \frac{\sigma_A^2}{N_i}\,\Gamma_i(N_i),
\qquad
\Gamma_i(N) \;:=\; 1 + 2\sum_{k=1}^{N-1}\Bigl(1 - \tfrac{k}{N}\Bigr)\rho_{i,k}.
\]
If $\sum_k |\rho_{i,k}| < \infty$, $\Gamma_i(N_i)$ converges to $c_i := 1 + 2\sum_{k\ge 1}\rho_{i,k}$, and the independent-noise minimizer extends to $N_i^\star = 2 S_i^2 / (2 S_i^2 - c_i\sigma_A^2)$ when the denominator is positive, meaning $2 S_i^2>c_i\sigma_A^2$.

\subsection{Reverse-Depth Has Two Mechanisms}

A tail-variance-only explanation is too strong. Two counterexamples show the issue.
\begin{itemize}
\item $Y_H \sim \mathcal{N}(10, 1)$, $Y_T \sim \mathcal{N}(0, 9)$: tail variance is larger, yet no reverse-depth occurs; the head remains dominant at any high threshold.
\item $Y_H \sim \mathcal{N}(0, 1)$, $Y_T \sim \mathcal{N}(1, 1)$: variances are equal but a tail mean shift already produces over-retention.
\end{itemize}
Tail variance explosion is one sufficient mechanism. Residual mean bias from the count curvature $r(N_i)$ is another independent one. Reverse-depth failure in practice draws on both; ablations should report them separately.

\parhead{Practical reading.}
Taken together, these derivations support three qualitative checks that appear in the main experiments.
A-only repair should remove the dominant prefix drift yet expose a recency failure.
Head participation should vary across architectures even when the query proxy is held fixed.
Multi-objective pooled retrieval should saturate earlier than factorized retention because the bottleneck there is channel count, not tail-window size.

\par\vspace{0.10em}
\begin{nolinenumbers}
\noindent\begin{minipage}{\textwidth}
\centering
\captionsetup{type=table,hypcap=false,skip=2pt}
\caption{Stage~I access-support comparison at $b=0.05$ on LongBench~\citep{bai2024longbench}. Each row applies the indicated repair on top of the previous row; scores are LongBench accuracy ($\times 100$), and $\Delta$ Avg reports the incremental change in the model average.}
\label{tab:confound_decomp}
\small
\setlength{\tabcolsep}{4pt}
\renewcommand{\arraystretch}{1.05}
\begin{threeparttable}
\begin{tabular*}{\textwidth}{@{\extracolsep{\fill}}>{\raggedright\arraybackslash}p{0.36\textwidth} S[table-format=2.1] S[table-format=2.1] S[table-format=2.1] S[table-format=2.1] S[table-format=+2.1]@{}}
\toprule
\textbf{Stage~I repair} & \textbf{Llama} & \textbf{Qwen3} & \textbf{Mistral} & \textbf{Avg} & \textbf{$\Delta$ Avg} \\
\midrule
Cumulative attention                 & 8.7  & 7.6  & 5.3  & 7.2  & {--} \\
$+$\,Count debiasing                 & 29.5 & 17.5 & 30.5 & 25.8 & +18.6 \\
$+$\,Decode-proximal queries         & 32.8 & 28.3 & 38.6 & 33.2 & +7.4  \\
$+$\,Layer denoising                 & 32.9 & 30.0 & 37.3 & 33.4 & +0.2  \\
\midrule
\fullkv{} reference                  & 35.7 & 31.6 & 39.5 & 35.6 & {--} \\
\bottomrule
\end{tabular*}
\begin{tablenotes}[flushleft]
\footnotesize
\item Avg averages Llama~\citep{grattafiori2024llama}, Qwen3~\citep{yang2025qwen3}, and Mistral~\citep{jiang2024mistral}. The repair sequence aligns with \Cref{sec:exposure_bias}: cumulative attention carries exposure-count bias; count debiasing removes the prefix over-counting; decode-proximal queries replace global aggregation with the SnapKV-style observation window; layer denoising restricts to the top-attention layer band.
\end{tablenotes}
\end{threeparttable}
\end{minipage}
\end{nolinenumbers}
\par\vspace{0.10em}

\subsection{Scope and Limitations}
\label{sec:appendix_exposure_bias_scope}

This appendix studies exposure-count bias under causal softmax and simple kernel and noise abstractions.
The derivations justify the direction of the residual diagnostics, but they do not assert that one parametric kernel explains every layer or task.
The main experimental claim relies on the measured residual mean, residual variance, and task-level controls, with the exponential surrogate used only as a tractable model.

\section{Head-Sum Dilution from Sparse Retrieval Heads}
\label{sec:appendix_head_aggregation}

This appendix derives the head-aggregation dilution used in \Cref{sec:residual}.
The mechanism compares two selector orderings: selecting keys at the head level and then combining retained candidates, or summing heads before ranking keys. A retrieval head may identify the right key, while the head-summed score can still dilute that signal with inactive-head background mass.

Let $H_{\mathrm{head}}$ be the number of heads and let $A_h(i)$ denote the locally aggregated attention mass that head $h$ assigns to candidate key $i$ inside a decode-proximal candidate set of size $n$. Here $A_h(i)$ is already a scalar mass, so it is not bolded. Center each head by the uniform background,
\[
\widetilde A_h(i) := A_h(i) - \frac{1}{n}.
\]
A per-head selector, or a union of per-head top-$k$ lists, is sensitive to the largest head-level contrast,
\[
S_{\infty}(i) := \max_h \widetilde A_h(i),
\]
whereas a head-summed selector ranks by
\[
S_{\Sigma}(i) := \sum_{h=1}^{H_{\mathrm{head}}} \widetilde A_h(i).
\]
Head-aggregation dilution is the contrast loss induced when a sparse head-level signal is passed through the head-summed score $S_{\Sigma}$.

\subsection{Sparse Retrieval Signal}

Let $i^\star$ be the target key and suppose that only a collection $\mathcal R$ of heads carries target-specific retrieval evidence, with $|\mathcal R|=r$. In the stylized retrieval model,
\[
\widetilde A_h(i^\star) =
\begin{cases}
\Delta, & h \in \mathcal R,\\
0, & h \notin \mathcal R,
\end{cases}
\]
where $\Delta>0$ is the per-active-head target contrast. A per-head selector sees target contrast $\Delta$. A head-summed selector sees target contrast $r\Delta$.

The difference appears in the nuisance background. Let $j\ne i^\star$ be a competing key and write its centered background as $B_{jh}=\widetilde A_h(j)$. Two endpoint models give the dilution factors.

\paragraph{Structured dense background.}
If the nuisance key receives approximately the same positive background $\tau$ from many heads, $B_{jh}\approx\tau$, then
\[
S_{\Sigma}(j)\approx H_{\mathrm{head}}\tau,
\qquad
S_{\Sigma}(i^\star)\approx r\Delta.
\]
Thus the head-summed contrast is
\[
\frac{S_{\Sigma}(i^\star)}{S_{\Sigma}(j)}
\approx
\frac{r}{H_{\mathrm{head}}}\cdot \frac{\Delta}{\tau}.
\]
The corresponding per-head contrast is $\Delta/\tau$. Therefore summing heads before selection loses a factor
\[
D_{\mathrm{struct}}
=
\frac{r}{H_{\mathrm{head}}}.
\]

\paragraph{Iid zero-mean background.}
If instead the nuisance background is zero-mean across heads with $\operatorname{Var}(B_{jh})=\tau^2$, then
\[
\sum_h B_{jh}
\]
has standard deviation $\tau\sqrt{H_{\mathrm{head}}}$ up to the usual extreme-value factor over nuisance keys. The target remains $r\Delta$, so the standardized head-summed contrast scales as
\[
\frac{r\Delta}{\tau\sqrt{H_{\mathrm{head}}}}
=
\frac{r}{\sqrt{H_{\mathrm{head}}}}
\cdot
\frac{\Delta}{\tau}.
\]
The iid-noise endpoint gives dilution
\[
D_{\mathrm{iid}}
=
\frac{r}{\sqrt{H_{\mathrm{head}}}}.
\]

The structured endpoint is harsher because the nuisance background adds coherently across heads. The iid endpoint is milder because nuisance fluctuations add only in standard deviation. Real models can mix both settings, so the two factors are diagnostic endpoints whose constants vary by model and task.

\subsection{Why the Active-Head Count Can Be Small}

The preceding bound assumes that retrieval evidence appears in only $r$ heads. QK-norm gives one reason why this can happen. Under QK-norm, query and key directions lie on a scaled unit sphere. Write the head-level logit as
\[
z_j = \sqrt{d_h}\, c_j,
\qquad
c_j := \langle \hat{\mathbf q}, \mathbf R_{\Delta_j}\hat{\mathbf k}_j\rangle
\in[-1,1].
\]
For the target key $i^\star$, define the angular gap
\[
\delta
:=
c_{i^\star} - \max_{j\ne i^\star} c_j.
\]
The logit margin is $m=\sqrt{d_h}\,\delta$. Since the non-target softmax mass is bounded by $(n-1)e^{-m}$, the target receives at least half the candidate mass when
\[
m \ge \log(n-1),
\]
or equivalently
\begin{equation}
\delta_{0.5}(n,d_h)
=
\frac{\log(n-1)}{\sqrt{d_h}} .
\label{eq:angular_threshold}
\end{equation}
Thus, as the candidate set grows, only heads whose learned geometry clears this margin can become majority-mass retrieval heads. The bound is a sufficient condition for strong head participation; successful retrieval can still occur below it. Its role here is to explain why $r$ can be much smaller than $H_{\mathrm{head}}$.

\subsection{Participation Ratio as the Measured Active-Head Count}

We estimate the effective active-head count with the participation ratio
\begin{equation}
r_{\mathrm{eff}}
:=
\frac{(\sum_h m_h)^2}{\sum_h m_h^2}
\le H_{\mathrm{head}},
\label{eq:participation}
\end{equation}
where
\[
m_h = \sum_{q\in \mathrm{last}\,64} A_h(q,i^\star).
\]
If the retrieval mass is equally distributed over $r$ heads and negligible elsewhere, then $r_{\mathrm{eff}}=r$. If the mass is uniform over all heads, then $r_{\mathrm{eff}}=H_{\mathrm{head}}$. Substituting $r_{\mathrm{eff}}$ for the idealized $r$ gives the empirical dilution diagnostics
\[
D_{\mathrm{struct}}
\approx
\frac{r_{\mathrm{eff}}}{H_{\mathrm{head}}},
\qquad
D_{\mathrm{iid}}
\approx
\frac{r_{\mathrm{eff}}}{\sqrt{H_{\mathrm{head}}}} .
\]
\Cref{fig:head_participation} reports $r_{\mathrm{eff}}/H_{\mathrm{head}}$ on the 8k NIAH retrieval check. Low participation means the retrieval signal is head-sparse, so head-summed scoring loses contrast relative to per-head selection.

\begin{figure}[H]
  \centering
  \includegraphics[width=0.88\linewidth]{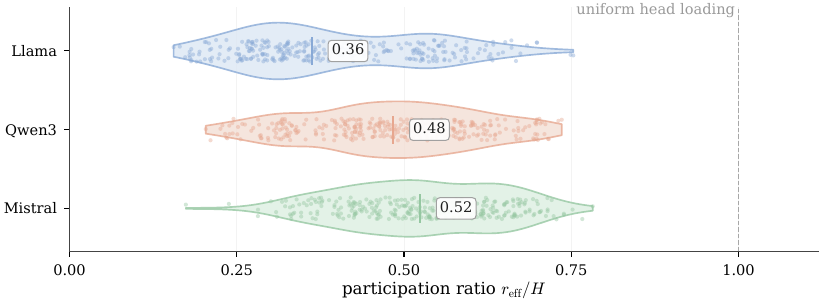}
  \caption{Effective head participation on the 8k NIAH retrieval check.}
  \label{fig:head_participation}
\end{figure}

\subsection{Consequence: Reserved Observation Window as a Keep Prior}

A fixed recent-token reserve can partly mask head-sum dilution near the decode tail.
Hard-reserving the last $W$ prefix tokens, as SnapKV~\citep{li2024snapkv} does with $W=32$, is equivalent at the decision level to adding an infinite keep prior on those indices.
Formally, for any finite scoring rule with kept score $S_i = C_i + \Lambda \cdot \mathbf{1}\{i \in \mathcal{W}\}$, if $|\mathcal{W}|\le K$ and $\Lambda$ exceeds the maximum score gap, top-$K$ on $S$ includes $\mathcal{W}$.
With a finite $\Lambda$, minimizing squared loss on the tail window picks $\Lambda^\star = -\beta_{\mathrm{tail}}$, giving a local MSE reduction of $\beta_{\mathrm{tail}}^2$.
Because the head-aggregation dilution factor $r_{\mathrm{eff}}/H_{\mathrm{head}}$ and the phase-mismatch factor $m/M_q$ both appear multiplicatively in $|\beta_{\mathrm{tail}}|$, reserved-window benefit scales quadratically with their product.
Under an iid or nonnegatively correlated noise endpoint, average pooling along the prefix axis, also used by SnapKV with kernel 7, contributes only a constant-factor variance reduction (iid gives $1/7$, perfectly correlated noise gives $1$); variance reduction alone cannot cancel an $O(1)$ bias.

\subsection{Consequence: Conjunctive Amplification on Multi-needle Retrieval}

Head-sum dilution becomes more visible when the task requires several evidence tokens to survive together.
The two-key needle-retrieval task grades correctness by conjunction over two needle tokens.
If each needle is retained independently with probability $p$, end-task accuracy scales as $p^2$, amplifying proxy cliffs multiplicatively.
This is why modest single-needle differences can become categorical end-task separations in multi-needle retrieval.
The independence calculation is a simple amplification model; correlated retention changes the exponent while preserving the qualitative fact that conjunction magnifies selection errors.

\subsection{Scope and Limitations}
\label{sec:appendix_head_aggregation_scope}

This appendix analyzes head aggregation as a contrast-loss mechanism under simplified retrieval and noise models.
The bounds explain when head-summed scoring degrades relative to per-head selection. Their scope is the local contrast loss; architecture-level comparisons require separate selector-contract tests.
We use the empirical participation ratio as a marker for the measured models.

\section{Layer-wise Analysis}
\label{sec:app_interpretability}

We vary the layer set used by the Module~I access-support measurement. The sweep starts from the top layer, adds lower layers, and re-evaluates against a leave-one-out classification reference. A brittle single-layer choice would produce a monotone alignment gain as lower layers are added. The observed curves move less across layer fractions than across model identities.

\subsection{Layer-Scope Sensitivity of Access Alignment}
\label{sec:app_i1_layer_inclusion_sweep}

The Module~I access measurement depends on which layers are included when aggregating attention-derived evidence. We sweep this layer scope as a local sensitivity check for the access measurement. \Cref{fig:i1_upperfrac_alignment} reports two complementary top-tail metrics. NDCG@5\% is rank-sensitive and asks whether reference-relevant positions are ordered near the top of the access score. J@5\% is support-sensitive and asks whether the top 5\% selected positions overlap with the reference top-support set.

Including lower layers does not yield a monotone improvement. Across the sweep, model-to-model differences are larger than the within-model changes induced by layer scope. Layer scope can affect individual access-support measurements; the main Module~II finding is measured under a fixed selector contract and replaces only the value-channel ranking slot.

\IfFileExists{figures/fig_i1_upperfrac_alignment_a.pdf}{%
\begin{figure}[!htbp]
\centering
\IfFileExists{figures/fig_i1_upperfrac_alignment_legend.pdf}{%
\includegraphics[width=0.27\linewidth]{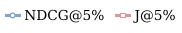}\par\vspace{-0.55em}
}{}
\makebox[\textwidth][c]{%
\begin{minipage}[t]{0.36\textwidth}
\centering
\includegraphics[width=\linewidth]{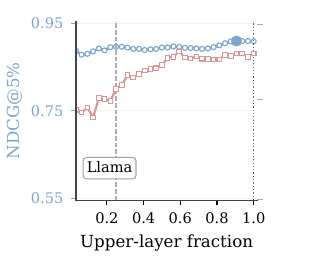}
\end{minipage}
\hspace{-0.040\textwidth}
\begin{minipage}[t]{0.36\textwidth}
\centering
\includegraphics[width=\linewidth]{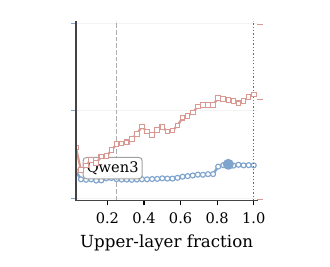}
\end{minipage}
\hspace{-0.040\textwidth}
\begin{minipage}[t]{0.36\textwidth}
\centering
\includegraphics[width=\linewidth]{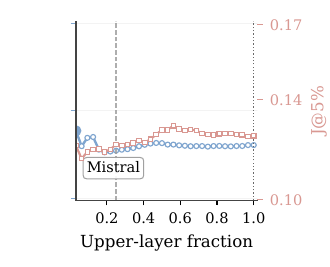}
\end{minipage}
}
	\caption{Layer-scope sensitivity of the Module~I access-support measurement. NDCG@5\% measures rank alignment in the top score tail, while J@5\% measures top-support overlap. Expanding the included layer range does not produce a monotone gain, and model-to-model variation is larger than within-model layer-scope variation.}
\label{fig:i1_upperfrac_alignment}
\end{figure}
}{}

\subsection{Scope and Limitations}
\label{sec:app_interpretability_scope}

The measurement is limited to the upper-layer inclusion rule used by the access-support calculation. Alternative reference definitions require a separate sensitivity check. The conclusion is local to this mechanism check and stops short of a general ranking of layer importance in long-context retrieval.

\section{Multi-Objective Retrieval Extended Results}
\label{sec:oos_prediction}

The multi-target retrieval results stress the support-closure prediction in \Cref{sec:ablation}. \Cref{sec:app_oos_mrt32k} reports the 32k-context NeedleBench M-RT evaluation~\citep{lineedlebench}, where pooled proxies lose support when targets split and question-routed proxies preserve more of it. \Cref{sec:app_ruler} reports an 8k-context RULER check~\citep{hsieh2024ruler} on whether logged movement concentrates on branched retrieval.

\subsection{NeedleBench M-RT at 32k context}
\label{sec:app_oos_mrt32k}

\Cref{tab:needlebench_mrt_32k_focus,tab:mrt_support_format} report the 32k-context NeedleBench M-RT stress test. The first table gives slot-set exact scores, with exact-hit rate in parentheses, under $n{=}4$ retrieval objectives, budget $b{=}0.10$, and 8 samples per model. The second separates retained-slot support from exact-format hits, distinguishing evidence recovery from answer-format recovery. All rows share the same prefill capture, context truncation, and budget contract. SnapKV~\citep{li2024snapkv} is the observation-window reference.

The Module~II value-channel replacement is evaluated in two pooling modes. MII-pooled uses the canonical pooled proxy bank $\mathcal U_{\mathrm{pool}}=\mathcal U^{\mathrm{tail}}$ and applies the same pooled block score used by the main Module~II recipe. MII-QD denotes question-routed Module~II. It sets $\mathcal U_{\mathrm{QD}}=\mathrm{DetectQuestionPositions}(x)$, falling back to $\mathcal U^{\mathrm{tail}}$ when no decoded-question span is detected, then uses the same pooled scoring and top-$k$ projection as MII-pooled.

MII-QD separates from SnapKV and MII-pooled. Averaged across the three models, it reaches $72.2$ slot-set exact; SnapKV and MII-pooled are at $61.1$ and $52.8$. Its exact-hit rate rises to $54.2$ from $25.0$, with the larger support-side gap pointing to evidence recovery beyond formatting.

The separation is consistent with the $r{=}1$ labeled-law TV floor in \Cref{thm:r_channel_tv}. The theorem lower-bounds labeled-law total variation; score ceilings are outside its scope. The result points to pooled-proxy failure when target branching grows. Single-channel baselines that change other contract axes require separate LongBench surfaces: H2O and H2O-debiased are access-support variants, while PyramidKV, Ada-KV, and CAKE change allocation. Their LongBench surfaces are evaluated in \Cref{tab:baseline_viability} and discussed in \Cref{sec:comparison_surfaces}.

\IfFileExists{tables/tab_needlebench_mrt_32k_focus.tex}{
\par\vspace{0.10em}
\begin{nolinenumbers}
\noindent\begin{minipage}{\textwidth}
\centering
\captionsetup{type=table,hypcap=false,skip=2pt}
\caption{Per-model split of the 32k NeedleBench M-RT~\citep{lineedlebench} stress slice. Slot-set exact and exact-hit rate are reported separately for each model.}
\label{tab:needlebench_mrt_32k_focus}
\scriptsize
\setlength{\tabcolsep}{3pt}
\renewcommand{\arraystretch}{1.04}
\begin{threeparttable}
\begin{tabular*}{\textwidth}{@{\extracolsep{\fill}}l S[table-format=3.1]S[table-format=2.1]S[table-format=2.1]S[table-format=2.1] S[table-format=2.1]S[table-format=2.1]S[table-format=2.1]S[table-format=2.1]@{}}
\toprule
& \multicolumn{4}{c}{\textbf{Slot-set exact (\%)}} & \multicolumn{4}{c}{\textbf{Exact-hit rate (\%)}} \\
\cmidrule(lr){2-5}\cmidrule(lr){6-9}
\textbf{Model} & \textbf{\fullkv{}} & \textbf{SnapKV} & \textbf{MII-pooled} & \textbf{MII-QD} & \textbf{\fullkv{}} & \textbf{SnapKV} & \textbf{MII-pooled} & \textbf{MII-QD} \\
\midrule
Qwen3-8B~\citep{yang2025qwen3} & 100.0 & 62.5 & 66.7 & 75.0 & 37.5 & 25.0 & 37.5 & 50.0 \\
Llama-3.1-8B~\citep{grattafiori2024llama} & 100.0 & 66.7 & 50.0 & 75.0 & 75.0 & 50.0 & 37.5 & 62.5 \\
Mistral-7B~\citep{jiang2024mistral} & 87.5  & 54.2 & 41.7 & 66.7 & 0.0  & 0.0  & 0.0  & 50.0 \\
\midrule
\textbf{Average} & 95.8 & 61.1 & 52.8 & 72.2 & 37.5 & 25.0 & 25.0 & 54.2 \\
\bottomrule
\end{tabular*}
\begin{tablenotes}[flushleft]
\footnotesize
\item The parenthesized exact-hit metric is split into its own numeric block so that support recovery and exact-format recovery can be read independently.
\end{tablenotes}
\end{threeparttable}
\end{minipage}
\end{nolinenumbers}
\par\vspace{0.10em}
}{}

\par\vspace{0.10em}
\begin{nolinenumbers}
\noindent\begin{minipage}{\textwidth}
\centering
\captionsetup{type=table,hypcap=false,skip=2pt}
\caption{Question-routed Module~II recovers the multi-target support that pooled proxies lose. Summarizes the averages expanded in \Cref{tab:needlebench_mrt_32k_focus}.}
\label{tab:mrt_support_format}
\footnotesize
\setlength{\tabcolsep}{4pt}
\renewcommand{\arraystretch}{1.06}
\begin{threeparttable}
\begin{tabular*}{\textwidth}{@{\extracolsep{\fill}}l S[table-format=2.1]S[table-format=2.1]S[table-format=2.1]S[table-format=2.1]@{}}
\toprule
\textbf{Metric} & \multicolumn{1}{c}{\makecell[c]{\textbf{reference}\\\fullkv{}}} & \multicolumn{1}{c}{\makecell[c]{\textbf{fixed contract}\\\textsc{SnapKV}}} & \multicolumn{1}{c}{\makecell[c]{\textbf{pooled proxy}\\\textsc{MII-pooled}}} & \multicolumn{1}{c}{\makecell[c]{\textbf{question-routed}\\\textsc{MII-QD}}} \\
\midrule
Slot-set exact (\%) & 95.8 & 61.1 & 52.8 & 72.2 \\
Exact-hit rate (\%) & 37.5 & 25.0 & 25.0 & 54.2 \\
\bottomrule
\end{tabular*}
\begin{tablenotes}[flushleft]
\footnotesize
\item The table keeps the original averages unchanged; only the column headers are rewritten so that each method name carries its evidential role.
\end{tablenotes}
\end{threeparttable}
\end{minipage}
\end{nolinenumbers}
\par\vspace{0.10em}

\par\vspace{0.15em}
\Needspace{6\baselineskip}
\refstepcounter{subsection}
\phantomsection
\noindent{\normalsize\bfseries\thesubsection\quad RULER branch-stress check at 8k context}
\par\nobreak\vspace{0.4ex}
\label{sec:app_ruler}

The RULER~\citep{hsieh2024ruler} branch-stress check is included as a directional 8k-context check within the LongBench budget, evaluated on the three 8B-class models in the main grid. The archived run stores three ordered score columns without semantic method names, so \Cref{tab:ruler_branch_stress} is used only to test whether the logged movement concentrates on branched rather than single-key retrieval.

Across the logged score columns, the positive movement is dominated by the two-key needle-retrieval task. On the single-key companion task, the logged changes are zero or near zero under the same budgets and model set. Removing the two-key task shrinks the average movement toward zero across all three models.

A pooled single-query proxy does not cover the multi-token query tail captured by the SnapKV observation window. The same direction holds at smaller context, so the binding constraint is the number of concurrent retrieval objectives. The 8k check is consistent with the M-RT direction but is not used to compare named selectors or estimate absolute multi-target retrieval ability.

\IfFileExists{tables/tab_ruler_branch_stress.tex}{
\par\vspace{0.10em}
\begin{nolinenumbers}
\noindent\begin{minipage}{\textwidth}
\centering
\captionsetup{type=table,hypcap=false,skip=2pt}
\caption{RULER~\citep{hsieh2024ruler} task breakdown for the 8k branch-stress check. The three logged scores are kept in their recorded order.}
\label{tab:ruler_branch_stress}
\scriptsize
\setlength{\tabcolsep}{3pt}
\renewcommand{\arraystretch}{1.03}
\begin{threeparttable}
\begin{tabular*}{\textwidth}{@{\extracolsep{\fill}}l S[table-format=+1.3] S[table-format=+1.3] S[table-format=+1.3] S[table-format=+1.3] S[table-format=+1.3] S[table-format=+1.3] S[table-format=+1.3] S[table-format=+1.3] S[table-format=+1.3]@{}}
\toprule
& \multicolumn{3}{c}{\textbf{Llama-3.1-8B~\citep{grattafiori2024llama}}} & \multicolumn{3}{c}{\textbf{Mistral-7B~\citep{jiang2024mistral}}} & \multicolumn{3}{c}{\textbf{Qwen3-8B~\citep{yang2025qwen3}}} \\
\cmidrule(lr){2-4}\cmidrule(lr){5-7}\cmidrule(lr){8-10}
\textbf{RULER task} & \textbf{score-1} & \textbf{score-2} & \textbf{score-3} & \textbf{score-1} & \textbf{score-2} & \textbf{score-3} & \textbf{score-1} & \textbf{score-2} & \textbf{score-3} \\
\midrule
niah\_single\_1 & -0.020 & -0.020 & +0.000 & +0.000 & +0.000 & +0.000 & +0.000 & +0.000 & +0.000 \\
\textbf{niah\_multikey\_2} & +0.920 & +0.640 & +0.400 & +0.420 & +0.440 & +0.400 & +0.660 & +0.740 & +0.760 \\
VT  & +0.140 & +0.048 & +0.044 & -0.008 & -0.008 & -0.004 & +0.176 & +0.120 & +0.060 \\
CWE & -0.004 & -0.020 & +0.080 & -0.248 & -0.052 & +0.138 & -0.016 & +0.070 & +0.162 \\
FWE & +0.193 & +0.147 & -0.007 & +0.080 & +0.007 & +0.027 & -0.033 & -0.007 & -0.033 \\
\midrule
\textbf{Average} & +0.246 & +0.159 & +0.103 & +0.049 & +0.077 & +0.112 & +0.157 & +0.185 & +0.190 \\
\textbf{Avg. excl. multikey} & +0.077 & +0.039 & +0.029 & -0.044 & -0.013 & +0.040 & +0.032 & +0.046 & +0.047 \\
\bottomrule
\end{tabular*}
\begin{tablenotes}[flushleft]
\footnotesize
\item Each cell was originally encoded as three ordered scores. The table keeps the logged order without assigning unsupported semantic names.
\end{tablenotes}
\end{threeparttable}
\end{minipage}
\end{nolinenumbers}
\par\vspace{0.10em}
}{}

\subsection{Scope and Limitations}
\label{sec:app_oos_scope}

The M-RT results show a directional separation between pooled and question-routed proxies under multi-target retrieval at 32k context; the RULER 8k check shows that logged movement concentrates on the multikey task. They do not identify the optimal proxy bank for arbitrary multi-target distributions, do not assign semantic names to the logged RULER score columns, and do not target absolute ceilings under retrieval-specific evaluators. The Stage~III diagnosis applies to contracts that fix the observation window, per-head budget allocation, and deterministic top-$k$ projection; transfer to allocation-changing contracts is reported on LongBench in \Cref{sec:app_pyramidkv_contract}.

\section{Conditional Split Provenance}
\label{sec:app_conditional_provenance}

\subsection{Stage Evidence Registry}
\label{sec:app_stage_registry}

\par\vspace{0.10em}
\begin{nolinenumbers}
\noindent\begin{minipage}{\textwidth}
\centering
\captionsetup{type=table,hypcap=false,skip=2pt}
\caption{Stage evidence registry for the fixed selector-contract diagnostic.}
\label{tab:stage_evidence_coverage}
\scriptsize
\setlength{\tabcolsep}{2pt}
\renewcommand{\arraystretch}{1.05}
\begin{tabular*}{\textwidth}{@{\extracolsep{\fill}}>{\raggedright\arraybackslash}p{0.13\textwidth}>{\raggedright\arraybackslash}p{0.21\textwidth}>{\raggedright\arraybackslash}p{0.27\textwidth}>{\raggedright\arraybackslash}p{0.20\textwidth}>{\raggedright\arraybackslash}p{0.13\textwidth}@{}}
\toprule
\textbf{Axis} & \textbf{Contract change} & \textbf{Evidence} & \textbf{Inference} & \textbf{Source} \\
\midrule
\textbf{Stage~II, $u_{\mathrm{val}}$} & value-channel replacement only & $H_c>0$: $72.6\%$, mean $+0.54$ pp & value repair under fixed support & \Cref{tab:mii_headroom} \\
Leverage sign comparison & fixed contract; sign agreement on 264 cells & NoLev sign matches MII on 239 of 264 cells ($90.5\%$); leverage on 162 of 264 cells ($61.4\%$) & value channel carries the predictive sign; leverage is conditional & \Cref{tab:stage_ablation} \\
Support coupling & fixed contract; stratify by $\phi$ & low $\phi$: $76.4\%$, $+0.69$ pp; high $\phi$: $42.9\%$, $-0.59$ pp & joint evidence can block repair & \Cref{tab:mii_phi_stratification,tab:phi_headroom_cross} \\
Stage~III, lattice & fixed block order; token-fill only & slack removed; mean $+0.12$ pp, median $0.00$ pp & token count is not support closure & \Cref{fig:stageiii_separability} \\
Stage~I, $u_{\mathrm{acc}}$ & change access-support law & count debias $7.2 \to 25.8$; decode-proximal $33.2$ & bottleneck moves upstream & \Cref{fig:exposure_tail_residual,fig:head_participation,fig:stagei_predictability}; \Cref{tab:confound_decomp} \\
Allocation drift & layer, cascade, or head budget changes & $+0.84$, $+0.30$, $+0.15$ pp & transfer under changed projection & \Cref{tab:baseline_viability} \\
Cross-contract comparison & PyramidKV-local reference margin & $H_B>0$: $79.5\%$, mean $+0.96$ pp & signed side transfers locally & \Cref{tab:baseline_conditional_gap} \\
Multi-target Stage~III & same capture; branched objective & pooled to QD: exact $52.8$ to $72.2$; exact-hit $25.0$ to $54.2$ & pooled score loses branch evidence & \Cref{tab:needlebench_mrt_32k_focus,tab:mrt_support_format,tab:ruler_branch_stress} \\
Negative controls & remove $T^u$; score boundary cases & NoLev reaches $1.70\times$ no-value ceiling & split depends on value channel & \Cref{tab:u_channel_ablation,tab:boundary_reference} \\
Leverage strength & same tasks and budgets; vary alpha & 96 completed cells; damped best counts 15, 10, and 7 & leverage is stage-resolved & \Cref{tab:p31_leverage_damping_summary} \\
Budget and reference surfaces & boundary-heavy four-budget diagnostics & 6 of 36 monotone; 12 of 36 local dips; 6 of 24 portable MII wins & averages hide phase and reference effects & \Cref{fig:budget_phase}; \Cref{tab:p30_boundary_summary} \\
\bottomrule
\end{tabular*}
\end{minipage}
\end{nolinenumbers}
\par\nopagebreak\vspace{0.05em}

\subsection{Predictor Provenance and Margin Split}
\label{sec:app_prereg_provenance}
\label{sec:app_confirmatory_tables}

The support-coupling index $\phi(x)$ comes from the fixed grammar in \Cref{sec:exposure_bias}, which counts code-line density, label-marker density, and delimiter-bounded record density. It measures the density of structural evidence units whose utility is not well represented by isolated token retention. We use the same index for the clean Module~II evaluation, the additive control, and the bounded-correction check. The signed reference margin is $m_c=\fullkv{}_c-\mathrm{Host}_c$, and the predictor is $H_c=\operatorname{sign}(m_c)$; in the SnapKV contract, $\mathrm{Host}_c=\mathrm{SnapKV}_c$. \Cref{tab:phi_headroom_cross} gives the raw $\phi(x)\times H_c$ split, and \Cref{tab:cluster_stats} treats task clusters as the resampling unit. The two predictors play different roles. The signed reference margin measures whether the value-channel replacement has room to help, while high support coupling identifies cells where a scalar value-channel correction is less reliable because useful states must be retained jointly.

\par\vspace{0.10em}
\begin{nolinenumbers}
\noindent\begin{minipage}{\textwidth}
\centering
\captionsetup{type=table,hypcap=false,skip=2pt}
\caption{Support coupling explains when positive reference margin still fails. Companion to \Cref{tab:mii_headroom}.}
\label{tab:mii_phi_stratification}
\scriptsize
\setlength{\tabcolsep}{2pt}
\renewcommand{\arraystretch}{1.06}
\begin{threeparttable}
\begin{tabular*}{\textwidth}{@{\extracolsep{\fill}}l S[table-format=2.0] S[table-format=2.1] S[table-format=+1.2] S[table-format=2.0] S[table-format=2.1] S[table-format=+1.2]@{}}
\toprule
& \multicolumn{3}{c}{\textbf{Positive margin $H_c>0$}} & \multicolumn{3}{c}{\textbf{Nonpositive margin $H_c\le 0$}} \\
\cmidrule(lr){2-4}\cmidrule(lr){5-7}
\textbf{$\phi$ bucket} & \textbf{Cells} & \textbf{Pos. shift (\%)} & \textbf{Mean \(\Delta\) (pp)} & \textbf{Cells} & \textbf{Pos. shift (\%)} & \textbf{Mean \(\Delta\) (pp)} \\
\midrule
low $\phi < 0.10$ & 55 & 76.4 & +0.69 & 23 & 34.8 & -0.49 \\
high $\phi \ge 0.10$ & 7 & 42.9 & -0.59 & 11 & 27.3 & -0.34 \\
\bottomrule
\end{tabular*}
\begin{tablenotes}[flushleft]
\footnotesize
\item The raw positive, nonpositive, and tie counts behind the same split are expanded in \Cref{tab:phi_headroom_cross}.
\end{tablenotes}
\end{threeparttable}
\end{minipage}
\end{nolinenumbers}
\par\nopagebreak\vspace{0.10em}

\par\vspace{0.10em}
\begin{nolinenumbers}
\noindent\begin{minipage}{\textwidth}
\centering
\captionsetup{type=table,hypcap=false,skip=2pt}
\caption{Support-coupling split under signed reference margin.}
\label{tab:phi_headroom_cross}
\scriptsize
\setlength{\tabcolsep}{2pt}
\renewcommand{\arraystretch}{1.06}
\begin{threeparttable}
\begin{tabular*}{\textwidth}{@{\extracolsep{\fill}}l l S[table-format=2.0] S[table-format=2.0] S[table-format=2.0] S[table-format=1.0] S[table-format=2.1] S[table-format=+1.2]@{}}
\toprule
\textbf{$\phi$ bucket} & \textbf{Margin side} & \textbf{Cells} & \textbf{Pos.} & \textbf{Nonpos.} & \textbf{Tie} & \textbf{Pos. shift (\%)} & \textbf{Mean \(\Delta\) (pp)} \\
\midrule
low $\phi < 0.10$ & $H_c > 0$ & 55 & 42 & 10 & 3 & 76.4 & +0.69 \\
high $\phi \ge 0.10$ & $H_c > 0$ & 7 & 3 & 3 & 1 & 42.9 & -0.59 \\
low $\phi < 0.10$ & $H_c \le 0$ & 23 & 8 & 9 & 6 & 34.8 & -0.49 \\
high $\phi \ge 0.10$ & $H_c \le 0$ & 11 & 3 & 8 & 0 & 27.3 & -0.34 \\
\bottomrule
\end{tabular*}
\begin{tablenotes}[flushleft]
\footnotesize
\item Counts split positive, nonpositive, and tied outcome cells before reporting the same rate-and-mean summary.
\end{tablenotes}
\end{threeparttable}
\end{minipage}
\end{nolinenumbers}
\par\vspace{0.10em}

\par\vspace{0.10em}
\begin{nolinenumbers}
\noindent\begin{minipage}{\textwidth}
\centering
\footnotesize
\setlength{\tabcolsep}{4pt}
\renewcommand{\arraystretch}{1.06}
\captionsetup{type=table,hypcap=false,skip=2pt}
\caption{Cluster-robust intervals for the signed reference-margin split. Confirms \Cref{tab:mii_headroom} when LongBench tasks are the resampling unit.}
\label{tab:cluster_stats}
\begin{threeparttable}
\begin{tabular*}{\textwidth}{@{\extracolsep{\fill}}>{\raggedright\arraybackslash}p{0.48\textwidth} l l@{}}
\toprule
\textbf{Statistic} & \textbf{Estimate} & \textbf{Task-cluster 95\% CI} \\
\midrule
$H_c>0$ positive-shift rate & $72.6\%$ & [58.2, 85.3] \\
$H_c\le 0$ positive-shift rate & $32.4\%$ & [19.5, 48.6] \\
Positive-rate gap $(H_c>0)-(H_c\le 0)$ & $+40.2$ pp & [$+18.9$, $+60.0$] \\
Direction-match rate, all 96 cells & $70.8\%$ & [60.4, 81.2] \\
Mean $\Delta$ gap $\overline{\Delta}_{H_c>0}-\overline{\Delta}_{H_c\le 0}$ & $+0.99$ pp & [$+0.57$, $+1.41$] \\
Pearson $r(m_c,\Delta)$ & $0.41$ & [$+0.22$, $+0.59$] \\
\midrule
\multicolumn{3}{l}{\emph{Threshold sweep on $|m_c|\ge\tau$}} \\
$\tau = 0$ & $+40.2$ pp & [$+18.9$, $+60.0$] \\
$\tau = 0.1$ pp & $+36.1$ pp & [$+17.2$, $+53.0$] \\
$\tau = 0.25$ pp & $+36.1$ pp & [$+18.3$, $+51.8$] \\
$\tau = 0.5$ pp & $+39.6$ pp & [$+20.0$, $+57.3$] \\
\midrule
\multicolumn{3}{l}{\emph{Leave-one-task-out extrema for the positive-rate gap}} \\
Worst-case dropped task & $+34.2$ pp & drop passage\_retrieval\_en \\
Best-case dropped task & $+46.0$ pp & drop samsum \\
\bottomrule
\end{tabular*}
\begin{tablenotes}[flushleft]
\footnotesize
\item The first block reports cluster-bootstrap estimates. The second and third blocks retain the threshold and leave-one-task-out checks at lower visual priority.
\end{tablenotes}
\end{threeparttable}
\end{minipage}
\end{nolinenumbers}
\par\vspace{0.10em}

\subsection{Robustness Checks for the Margin Split}
\label{sec:app_headroom_robustness_audits}

\Cref{tab:permutation_null,tab:leave_one_task_out} check whether the signed reference-margin gap comes from label counts or from a single task. The permutation null keeps the same 96 matched outcomes and the same positive and nonpositive split, then shuffles the labels. The leave-one-task-out table deletes each LongBench task in turn. The observed direction-match and positive-rate gaps lie outside the permutation null, and the signed reference-margin gap remains positive after every one-task deletion. The support-coupling controls below add a granularity check. Task-level resampling preserves the high-$\phi$ direction, while prompt-level splitting weakens once the task template is fixed.

\par\vspace{0.10em}
\begin{nolinenumbers}
\noindent\begin{minipage}{\textwidth}
\centering
\footnotesize
\setlength{\tabcolsep}{4pt}
\renewcommand{\arraystretch}{1.08}
\captionsetup{type=table,hypcap=false,skip=2pt}
\caption{Permutation-null check for the fixed signed reference-margin partition. Confirms that the observed split is not reproduced by label shuffling on the same 96 cells.}
\label{tab:permutation_null}
\begin{threeparttable}
\begin{tabular*}{\textwidth}{@{\extracolsep{\fill}}>{\raggedright\arraybackslash}p{0.36\textwidth} S[table-format=+2.1] S[table-format=+2.1] l c@{}}
\toprule
\textbf{Statistic} & \textbf{Observed} & \textbf{Null mean} & \textbf{Null 95\% interval} & \textbf{One-sided $p$} \\
\midrule
Direction-match rate (\%) & 70.8 & 52.5 & [43.8, 62.5] & $< 0.001$ \\
Positive-rate gap $(H_c>0)-(H_c\le0)$ (pp) & +40.2 & +0.1 & [$-19.0$, $+22.0$] & $< 0.001$ \\
\bottomrule
\end{tabular*}
\begin{tablenotes}[flushleft]
\footnotesize
\item The null center and null interval are separated so that displacement and dispersion remain distinct.
\end{tablenotes}
\end{threeparttable}
\end{minipage}
\end{nolinenumbers}
\par\vspace{0.10em}

\par\vspace{0.10em}
\begin{nolinenumbers}
\noindent\begin{minipage}{\textwidth}
\centering
\scriptsize
\setlength{\tabcolsep}{2pt}
\renewcommand{\arraystretch}{1.07}
\captionsetup{type=table,hypcap=false,skip=2pt}
\caption{Leave-one-task-out stability of the signed reference-margin split. Confirms \Cref{tab:mii_headroom} after deleting each LongBench task in turn.}
\label{tab:leave_one_task_out}
\begin{threeparttable}
\begin{tabular*}{\textwidth}{@{\extracolsep{\fill}}l S[table-format=2.1] S[table-format=2.1] S[table-format=+2.1] S[table-format=+1.2]@{}}
\toprule
\textbf{Dropped task} & \textbf{$H_c>0$ pos. (\%)} & \textbf{$H_c\le0$ pos. (\%)} & \textbf{Rate gap (pp)} & \textbf{Mean gap (pp)} \\
\midrule
qasper                 & 69.6 & 32.4 & +37.3 & +0.88 \\
narrativeqa            & 70.7 & 31.2 & +39.4 & +0.96 \\
multifieldqa\_en      & 72.4 & 34.4 & +38.0 & +0.97 \\
hotpotqa               & 70.7 & 31.2 & +39.4 & +0.97 \\
2wikimqa               & 72.4 & 31.2 & +41.2 & +1.04 \\
musique                & 75.4 & 33.3 & +42.1 & +0.94 \\
gov\_report           & 69.6 & 32.4 & +37.3 & +0.94 \\
qmsum                  & 73.7 & 30.3 & +43.4 & +1.03 \\
multi\_news           & 69.6 & 32.4 & +37.3 & +0.94 \\
samsum                 & 74.1 & 28.1 & +46.0 & +1.06 \\
trec                   & 74.6 & 32.3 & +42.3 & +0.91 \\
triviaqa               & 74.6 & 32.3 & +42.3 & +1.04 \\
passage\_count        & 74.6 & 32.3 & +42.3 & +1.05 \\
passage\_retrieval\_en & 72.1 & 37.9 & +34.2 & +0.92 \\
lcc                    & 73.3 & 33.3 & +40.0 & +1.03 \\
repobench-p            & 73.3 & 33.3 & +40.0 & +1.05 \\
\midrule
\textbf{minimum gap across LOO} & 72.1 & 37.9 & +34.2 & +0.92 \\
\textbf{maximum gap across LOO} & 74.1 & 28.1 & +46.0 & +1.06 \\
\bottomrule
\end{tabular*}
\begin{tablenotes}[flushleft]
\footnotesize
\item Each row drops six cells from the 96-cell grid ($3$ models $\times$ $2$ budgets) for the named LongBench task, matching the source-file comment.
\end{tablenotes}
\end{threeparttable}
\end{minipage}
\end{nolinenumbers}
\par\vspace{0.10em}

\refstepcounter{subsection}
\phantomsection
\noindent{\normalsize\bfseries\thesubsection\quad Support-Coupling and Margin Controls}
\par\nobreak\vspace{0.8ex}
\label{sec:app_phi_headroom_tables}
\label{sec:app_phi_within_task}

\Cref{tab:phi_headroom_cross,tab:phi_granularity} separate cell-level and prompt-level behavior. High-$\phi$ cells have a lower positive-shift rate than low-$\phi$ cells overall (one-sided Fisher $p=0.0173$). Within positive-margin cells, high-$\phi$ reaches $42.9\%$ while low-$\phi$ reaches $76.4\%$ ($1.8{\times}$; one-sided Fisher $p=0.0823$). The cell-level gap is $-33.5$\,pp on all tasks, $-43.0$\,pp after removing code tasks, and $-42.6$\,pp after also removing retrieval tasks.

\Cref{tab:phi_granularity} gives the prompt-level counterpart on positive-margin prompt-budget units. The direction is the same but smaller: $-13.9$\,pp overall, $-4.7$\,pp after removing code, and $-4.8$\,pp after also removing retrieval. This scale difference follows from the definition of $\phi(x)$. Task templates create large between-family variation at cell granularity, while most non-code tasks have little within-task variation. Code contributes many high-$\phi$ prompt units, so removing code reduces the prompt-level effect while leaving the cell-level direction intact. The main $\phi$ effect sits at task and template granularity, with a weaker same-direction prompt-level trace.

\par\vspace{0.10em}
\begin{nolinenumbers}
\noindent\begin{minipage}{\textwidth}
\centering
\captionsetup{type=table,hypcap=false,skip=2pt}
\caption{Within-family and prompt-level aggregations both preserve the $\phi$ direction. Companion to \Cref{tab:phi_headroom_cross} at family granularity.}
\label{tab:phi_granularity}
\footnotesize
\setlength{\tabcolsep}{3pt}
\renewcommand{\arraystretch}{1.06}
\begin{threeparttable}
\begin{tabular*}{\textwidth}{@{\extracolsep{\fill}}l l S[table-format=4.0] S[table-format=2.1] S[table-format=2.1] S[table-format=+2.1]@{}}
\toprule
\textbf{Granularity} & \textbf{Slice} & \textbf{Units} & \textbf{Low-$\phi$ pos. (\%)} & \textbf{High-$\phi$ pos. (\%)} & \textbf{Gap (pp)} \\
\midrule
Cell & All tasks & 62 & 76.4 & 42.9 & -33.5 \\
Cell & Non-code only & 58 & 76.4 & 33.3 & -43.0 \\
Cell & Non-code, non-retrieval & 57 & 75.9 & 33.3 & -42.6 \\
Prompt level & All tasks & 2743 & 57.4 & 43.4 & -13.9 \\
Prompt level & Non-code only & 2335 & 57.4 & 52.6 & -4.7 \\
Prompt level & Non-code, non-retrieval & 2327 & 57.4 & 52.6 & -4.8 \\
\bottomrule
\end{tabular*}
\begin{tablenotes}[flushleft]
\footnotesize
\item Negative gaps mean the high-$\phi$ bucket has a lower positive-shift rate.
\end{tablenotes}
\end{threeparttable}
\end{minipage}
\end{nolinenumbers}
\par\vspace{0.10em}

\par\vspace{0.15em}
\refstepcounter{subsection}
\phantomsection
\noindent{\normalsize\bfseries\thesubsection\quad Cross-Contract Transfer and Identity Checks}
\par\nobreak\vspace{0.4ex}
\label{sec:app_pyramidkv_contract}
\label{sec:app_identity_checks}

\Cref{tab:app_pyramidkv_contract,tab:baseline_conditional_gap} report cross-contract transfer and identity coverage. \Cref{tab:app_pyramidkv_contract} separates the contract-local reference margin, adapter shifts, and outcome counts by model-task pair. The PyramidKV~\citep{cai2024pyramidkv} evaluation tests whether the positive-margin prediction survives a changed allocation contract. Over the full 96-cell grid, using the contract-local reference margin $m_B=\fullkv{}-\mathrm{PyramidKV}$ with $H_B=\operatorname{sign}(m_B)$, PyramidKV+MII is positive on $79.5\%$ of positive-margin cells with a mean $+0.96$\,pp lift. The SnapKV~\citep{li2024snapkv} contract result in \Cref{sec:exp_module2} has the same direction, with a $72.6\%$ positive-shift rate at mean $+0.54$\,pp.

\par\vspace{0.10em}
\begin{nolinenumbers}
\noindent\begin{minipage}{\textwidth}
\centering
\captionsetup{type=table,hypcap=false,skip=2pt}
\caption{Cross-instantiation breakdown for the PyramidKV transfer rows.}
\label{tab:app_pyramidkv_contract}
\footnotesize
\setlength{\tabcolsep}{3pt}
\renewcommand{\arraystretch}{1.05}
\begin{threeparttable}
\begin{tabular*}{\textwidth}{@{\extracolsep{\fill}}l l S[table-format=+2.2] S[table-format=+1.2] S[table-format=+1.2] c S[table-format=2.0] S[table-format=2.0] S[table-format=2.0]@{}}
\toprule
& & \multicolumn{3}{c}{\textbf{Contract-local comparison (pp)}} & & \multicolumn{3}{c}{\textbf{Outcome counts}} \\
\cmidrule(lr){3-5}\cmidrule(lr){7-9}
\textbf{Model} & \textbf{Task} & \textbf{\fullkv{}-Pyr} & \textbf{MII-Pyr} & \textbf{MII-Snap} & \textbf{Margin side} & \textbf{Positive} & \textbf{Nonpositive} & \textbf{Tie} \\
\midrule
\multirow{4}{*}{Llama} & qasper & +9.25 & -0.04 & +1.64 & positive & 30 & 32 & 38 \\
& qmsum & +0.85 & +0.09 & +0.07 & positive & 45 & 42 & 13 \\
& trec & +13.00 & -1.00 & 0.00 & positive & 1 & 2 & 97 \\
& lcc & -0.60 & +0.01 & +0.30 & nonpositive & 11 & 5 & 84 \\
\midrule
\multirow{4}{*}{Mistral} & qasper & +10.56 & +1.05 & +1.51 & positive & 33 & 25 & 42 \\
& qmsum & +0.62 & +0.60 & +0.29 & positive & 52 & 35 & 13 \\
& trec & -29.00 & +3.00 & +1.00 & nonpositive & 10 & 7 & 83 \\
& lcc & +8.36 & +0.33 & -2.84 & positive & 22 & 34 & 44 \\
\midrule
\multirow{4}{*}{Qwen3} & qasper & +9.32 & -0.07 & +2.76 & positive & 34 & 25 & 41 \\
& qmsum & +1.24 & +0.31 & -0.19 & positive & 52 & 40 & 8 \\
& trec & +9.50 & +9.50 & -5.00 & positive & 12 & 3 & 85 \\
& lcc & +1.02 & +0.60 & -0.02 & positive & 50 & 35 & 15 \\
\bottomrule
\end{tabular*}
\begin{tablenotes}[flushleft]
\footnotesize
\item Positive, nonpositive, and tie counts expose the transfer breakdown for each model-task pair.
\end{tablenotes}
\end{threeparttable}
\end{minipage}
\end{nolinenumbers}
\par\vspace{0.10em}

The nonpositive side is asymmetric across contracts. Under SnapKV, the nonpositive margin predicts the weak group, with a $32.4\%$ positive-shift rate and a mean shift of $-0.44$\,pp. Under PyramidKV, the corresponding bucket gives a $56.5\%$ positive-shift rate and a mean $+0.45$\,pp shift. This side remains inconclusive. PyramidKV changes layer allocation and projection at the same time, and the non-monotone-compression observation in \Cref{sec:intro} makes \fullkv{} an unreliable upper bound (SnapKV crosses \fullkv{} on 23 of 96 cells, concentrated in 9 of 18 Few-shot cells and 8 of 12 Code cells); as a result, $H_B\le 0$ gives a weak saturation control for Stage~II error.

The identity checks rule out selector reimplementation as the source of the positive transfer. The additive path keeps the identity scoring branch separate from the value-channel branch, and local unit tests cover exact SnapKV score and selection preservation. For PyramidKV, the $\alpha{=}0$ adapter was checked end-to-end on two cross-instantiation cells at $b{=}0.05$; both produced equal aggregate scores under PyramidKV and the PyramidKV value-channel adapter with $\alpha{=}0$ (deltas $0.0$).

\Cref{tab:baseline_conditional_gap} separates the same cross-contract comparison into conditional rows, aggregate rows, and no-adapter surfaces for SnapKV, PyramidKV, CAKE~\citep{qin2025cake}, Ada-KV~\citep{feng2026adakv}, and H2O~\citep{zhang2023ho}.

\par\vspace{0.10em}
\begin{nolinenumbers}
\noindent\begin{minipage}{\textwidth}
\centering
\captionsetup{type=table,hypcap=false,skip=2pt}
\caption{Per-surface breakdown of \Cref{tab:baseline_viability}.}
\label{tab:baseline_conditional_gap}
\scriptsize
\setlength{\tabcolsep}{2pt}
\renewcommand{\arraystretch}{1.04}
\begin{threeparttable}
\begin{tabular*}{\textwidth}{@{\extracolsep{\fill}}l l l l S[table-format=2.1] c p{0.15\textwidth}@{}}
\toprule
\textbf{Surface $B$} & \textbf{Contract axis} & \textbf{Row type} & \textbf{Rate kind} & \textbf{Rate (\%)} & \textbf{Mean \(\Delta_B\) (pp)} & \textbf{Use} \\
\midrule
\multirow{2}{*}{SnapKV} & \multirow{2}{*}{observation-window scoring} & $H_c>0$ & positive & 72.6 & +0.54 & main validation \\
 &  & $H_c\le 0$ & nonpositive & 67.6 & -0.44 &  \\
\midrule
\multirow{2}{*}{PyramidKV} & \multirow{2}{*}{layer-budget projection} & $H_B>0$ & positive & 79.5 & +0.96 & positive-margin transfer \\
 &  & $H_B\le 0$ & positive & 56.5 & +0.45 &  \\
\midrule
\multirow{2}{*}{CAKE} & \multirow{2}{*}{layer cascade} & all 96 cells & positive & 60.4 & +0.29 & Stage~III surface \\
 &  & NoLev aggregate & \multicolumn{1}{c}{--} & \multicolumn{1}{c}{--} & +0.34 &  \\
\midrule
\multirow{2}{*}{Ada-KV} & \multirow{2}{*}{head-budget allocation} & all 96 cells & positive & 53.1 & +0.16 & Stage~III surface \\
 &  & NoLev aggregate & \multicolumn{1}{c}{--} & \multicolumn{1}{c}{--} & +0.26 &  \\
\midrule
H2O-deb. & count-debiased cumulative support & no adapter row & \multicolumn{1}{c}{--} & \multicolumn{1}{c}{--} & \multicolumn{1}{c}{--} & Stage~I surface \\
H2O & cumulative support & no adapter row & \multicolumn{1}{c}{--} & \multicolumn{1}{c}{--} & \multicolumn{1}{c}{--} & Stage~I surface \\
\bottomrule
\end{tabular*}
\begin{tablenotes}[flushleft]
\footnotesize
\item $m_B=\fullkv{}-B$ is the contract-local reference margin for surface $B$, and $H_B=\operatorname{sign}(m_B)$ gives the split. The rate-kind column states whether each percentage counts positive or nonpositive shifts; the older paragraph-style cells left this distinction implicit.
\end{tablenotes}
\end{threeparttable}
\end{minipage}
\end{nolinenumbers}
\par\vspace{0.10em}

\par\vspace{0.15em}
\Needspace{12\baselineskip}
\refstepcounter{subsection}
\phantomsection
\noindent{\normalsize\bfseries\thesubsection\quad Stage~I Predictability Under Count Debiasing}
\par\nobreak\vspace{0.4ex}
\label{sec:app_stagei_predictability}

\begin{nolinenumbers}
\noindent\begin{minipage}{\textwidth}
\captionsetup{type=figure,hypcap=false,skip=3pt,font=small}
\centering
\includegraphics[width=0.54\linewidth]{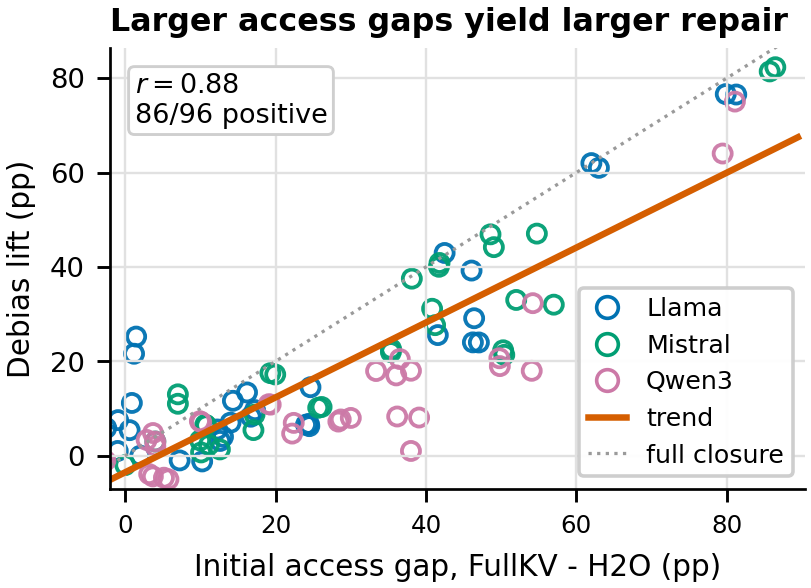}
\caption{Stage~I access gaps predict count-debias repair.}
\label{fig:stagei_predictability}
\end{minipage}
\end{nolinenumbers}
\par\nopagebreak\vspace{0.15em}

Count debiasing of cumulative attention is established prior work~\citep{gu2025ahakv,zhao2025attentiondebiasing}; the contribution here is the predictability contrast across stages. \Cref{fig:stagei_predictability} plots the per-cell debias lift against the initial H2O--\fullkv{} access gap on the LongBench grid. Each point denotes one grid cell, and the scatter has a strong linear trend (correlation $r{=}0.88$, 86 of 96 cells positive, mean lift $+18.4$\,pp), consistent on this 96-cell grid with count-debias repair scaling with the measured Stage~I access gap. The same scatter geometry under Stage~III (\Cref{fig:stageiii_separability}c) gives $r{=}0.14$, the opposite predictability signature. Within-stage repair under count debiasing has been characterized by AhaKV and Attention Debiasing; the cross-stage predictability contrast shown here and in panel (c) of \Cref{fig:stageiii_separability} has not, to our knowledge, been reported.

\section{Additional Confirmatory Controls}
\label{sec:app_additional_controls}

These controls test two limits of the conditional split. One removes the value-consequence term from the additive scalar; the other replaces task accuracy with a boundary-local reference based on next-token distribution change. \Cref{tab:u_channel_ablation,tab:boundary_reference} evaluate the value-consequence channel and the signed reference-margin prediction at disagreement boundaries. Two stress analyses then probe leverage strength and boundary-heavy budget behavior without changing the main 96-cell LongBench grid.

\subsection{Value-Consequence Negative Control}
\label{sec:app_value_channel_negative_control}

\Cref{tab:u_channel_ablation} removes the value-consequence channel from the additive scalar while keeping the access and contract terms. Here $T^a$, $T^c$, and $T^u$ denote the additive access, contract, and value-consequence terms. The within-cell comparison uses the same SnapKV baseline across rows. The channel-positive NoLev reference gives a mean $+0.45$\,pp shift and a $68.9\%$ positive-shift rate, $1.70{\times}$ the best no-value row. The three variants without $T^u$ stay near zero or below zero.

\begin{nolinenumbers}
\noindent\begin{minipage}{\textwidth}
\centering
\captionsetup{type=table,hypcap=false,skip=2pt}
\caption{Removing the value-consequence channel collapses lift toward the no-value ceiling. Companion to \Cref{tab:stage_ablation}.}
\label{tab:u_channel_ablation}
\footnotesize
\setlength{\tabcolsep}{4pt}
\renewcommand{\arraystretch}{1.08}
\begin{threeparttable}
\begin{tabular*}{\textwidth}{@{\extracolsep{\fill}}l c S[table-format=+1.2] S[table-format=2.1] l@{}}
\toprule
\textbf{Method} & \textbf{Retains \(T^u\)} & \textbf{Mean \(\Delta\) from SnapKV (pp)} & \textbf{Positive shift (\%)} & \textbf{Rate / no-value ceiling} \\
\midrule
SnapKV + NoLev & yes & +0.45 & 68.9 & $1.70\times$ \\
\midrule
Addi-$a$, $T^a$ only & no & -0.03 & 39.2 & $0.97\times$ \\
Addi-$c$, $T^c$ only & no & -0.10 & 35.1 & $0.87\times$ \\
Addi-$a{+}c$, $T^a{+}T^c$ & no & -0.04 & 40.5 & $1.00\times$ \\
\bottomrule
\end{tabular*}
\begin{tablenotes}[flushleft]
\footnotesize
\item The first row is the channel-positive reference. The remaining rows are all no-value ablations from the same additive family; exposing the $T^u$ column makes the control axis readable without consulting the surrounding paragraph.
\end{tablenotes}
\end{threeparttable}
\end{minipage}
\end{nolinenumbers}

\par\vspace{0.15em}
\refstepcounter{subsection}
\phantomsection
\noindent{\normalsize\bfseries\thesubsection\quad Boundary-Level Split Controls}
\par\nobreak\vspace{0.4ex}
\label{sec:app_boundary_split_controls}

\Cref{tab:boundary_reference} scores each disagreement boundary by its effect on the next-token distribution under \fullkv{}. Spearman~$\rho$ and pairwise AUC stay near chance for both selectors, indicating weak global ranking alignment. After stratification, the swap-improvement rate rises in the predicted-positive quadrants and stays near $0.5$ in the predicted-nonpositive quadrants, matching the signed reference-margin diagnosis at the boundary level.

The prediction is strongest when the signed reference margin and boundary-level value consequences agree. It weakens when the contract leaves little useful value-channel error to repair, or when high support coupling makes the correction less reliable.

\begin{nolinenumbers}
\noindent\begin{minipage}{\textwidth}
\centering
\scriptsize
\setlength{\tabcolsep}{3pt}
\renewcommand{\arraystretch}{1.04}
\captionsetup{type=table,hypcap=false,skip=2pt}
\caption{Boundary-local alignment rises only in the predicted-positive quadrants under the fixed observation-window contract. Complements \Cref{tab:u_channel_ablation} with disagreement-boundary diagnostics.}
\label{tab:boundary_reference}
\begin{threeparttable}
\begin{tabular*}{\textwidth}{@{\extracolsep{\fill}}l S[table-format=1.2] S[table-format=1.2] S[table-format=1.2] S[table-format=1.2] S[table-format=1.2] S[table-format=1.2] S[table-format=1.2] S[table-format=2.0]@{}}
\toprule
& \multicolumn{4}{c}{\textbf{Spearman $\rho$ with KL reference}} & \multicolumn{3}{c}{\textbf{Pairwise AUC}} & \\
\cmidrule(lr){2-5}\cmidrule(lr){6-8}
\textbf{Group} & \textbf{SnapKV} & \textbf{Support} & \textbf{NoLev} & \textbf{MII} & \textbf{SnapKV} & \textbf{Support} & \textbf{MII} & \textbf{Swap-improving boundaries (\%)} \\
\midrule
All reference prompts & 0.05 & 0.08 & 0.10 & 0.12 & 0.52 & 0.55 & 0.58 & 58 \\
\addlinespace[1pt]
$H_c > 0$ & 0.06 & 0.10 & 0.13 & 0.15 & 0.53 & 0.57 & 0.61 & 62 \\
low-$\phi$, $H_c > 0$ & 0.06 & 0.10 & 0.14 & 0.17 & 0.53 & 0.58 & 0.65 & 65 \\
\addlinespace[1pt]
$H_c \le 0$ & 0.08 & 0.09 & 0.09 & 0.09 & 0.52 & 0.53 & 0.53 & 49 \\
high-$\phi$, $H_c > 0$ & 0.07 & 0.07 & 0.07 & 0.07 & 0.53 & 0.52 & 0.52 & 50 \\
high-$\phi$, $H_c \le 0$ & 0.07 & 0.07 & 0.07 & 0.06 & 0.53 & 0.52 & 0.52 & 48 \\
\bottomrule
\end{tabular*}
\begin{tablenotes}[flushleft]
\footnotesize
\item The final column reports the percentage of disagreement boundaries whose swap lowers KL under the \fullkv{} reference. The global ranking columns remain near chance; the local improvement signal appears only after stratification.
\end{tablenotes}
\end{threeparttable}
\end{minipage}
\end{nolinenumbers}

\par\vspace{0.15em}
\refstepcounter{subsection}
\phantomsection
\noindent{\normalsize\bfseries\thesubsection\quad Targeted Leverage-Strength Sweep}
\par\nobreak\vspace{0.4ex}
\label{sec:app_p31_leverage_damping}

The NoLev-to-MII decomposition treats leverage as a conditional correction. A targeted leverage-strength sweep varies the blending coefficient while keeping tasks, budgets, and methods fixed. \Cref{tab:p31_leverage_damping_summary} summarizes the 96 completed cells for Qwen3~\citep{yang2025qwen3} and Llama~\citep{grattafiori2024llama}.

The sweep does not support a single alpha rule. Across the 32 model--task--budget contracts, the best damped alpha splits across the three tested values with counts 15, 10, and 7 under a lowest-alpha tie break. When the matched $\alpha=1$ reference is included, it remains best in 11 contracts. Absolute and matched-reference conclusions also diverge. 18 of 96 matched cells are positive over SnapKV but worse than $\alpha=1$, while 13 improve over $\alpha=1$ but still lose to SnapKV. These cases make leverage strength a stage-resolved validation surface.

\begin{nolinenumbers}
\noindent\begin{minipage}{\textwidth}
\centering
\captionsetup{type=table,hypcap=false,skip=2pt}
\caption{Leverage-strength sweep over 96 completed cells.}
\label{tab:p31_leverage_damping_summary}
\scriptsize
\setlength{\tabcolsep}{2.2pt}
\renewcommand{\arraystretch}{1.04}
\begin{threeparttable}
\begin{tabular*}{\textwidth}{@{\extracolsep{\fill}}lcccccccccccc@{}}
\toprule
\textbf{Model} &
\textbf{Alpha} &
\textbf{Cells} &
\makecell{\textbf{Raw}\\\textbf{MII}} &
\textbf{Wins} &
\textbf{Losses} &
\textbf{Ties} &
\textbf{NoLev} &
\textbf{Leverage} &
\textbf{Rescues} &
\textbf{Breaks} &
\textbf{Taxes} &
\makecell{\textbf{Support}\\\textbf{failures}} \\
\midrule
Llama & 0.25 & 16 & +0.015 & 8 & 5 & 3 & -0.158 & +0.173 & 2 & 1 & 1 & 6 \\
Llama & 0.50 & 16 & -0.622 & 5 & 10 & 1 & -0.313 & -0.309 & 1 & 2 & 2 & 9 \\
Llama & 0.75 & 16 & -0.481 & 7 & 8 & 1 & -0.435 & -0.046 & 0 & 0 & 5 & 9 \\
Qwen3 & 0.25 & 16 & +0.168 & 8 & 6 & 2 & +0.115 & +0.053 & 1 & 1 & 3 & 7 \\
Qwen3 & 0.50 & 16 & +0.450 & 9 & 5 & 2 & +0.464 & -0.014 & 0 & 1 & 4 & 6 \\
Qwen3 & 0.75 & 16 & +0.414 & 8 & 6 & 2 & +0.887 & -0.473 & 0 & 3 & 5 & 5 \\
\bottomrule
\end{tabular*}
\begin{tablenotes}[flushleft]
\footnotesize
\item Raw MII and NoLev are mean percentage-point gains over SnapKV. Leverage is Raw MII minus NoLev. Breaks are sign-flipping leverage failures; taxes are support-win cases where leverage reduces an otherwise positive NoLev result.
\end{tablenotes}
\end{threeparttable}
\end{minipage}
\end{nolinenumbers}

\par\vspace{0.35em}
\refstepcounter{subsection}
\phantomsection
\noindent{\normalsize\bfseries\thesubsection\quad Boundary-Heavy Diagnostic Subset}
\par\nobreak\vspace{0.4ex}
\label{sec:app_p30_failuremix_diagnostics}

We also evaluate boundary-heavy diagnostic cases as a stress test for the same fixed-contract mechanism. Their role is to expose reference-margin and budget-axis failures around the selector boundary, where the completed LongBench grid already shows that final averages can hide stage-specific behavior.

\Cref{fig:budget_phase} shows that retained budget behaves like a phase variable in these cases. Raw MII is monotone nondecreasing in only 6 of 36 complete four-budget trajectories, and budget \(0.05\) is a strict local dip in 12 of 36. Budget \(0.10\) is the strongest fixed default, with \(+1.066\) pp, 22 wins, 10 losses, and 4 ties, but it captures only \(55.5\%\) of the per-trajectory oracle-best mean and is exact-best or tied in only 16 of 36 trajectories. The largest adjacent transitions occur in Qwen3 TREC, Qwen3 passage retrieval, and Mistral TREC, with ranges from \(6.250\) to \(13.281\) pp.

\begin{nolinenumbers}
\noindent\begin{minipage}{\textwidth}
\captionsetup{type=figure,hypcap=false,skip=3pt}
\centering
\includegraphics[width=0.96\linewidth]{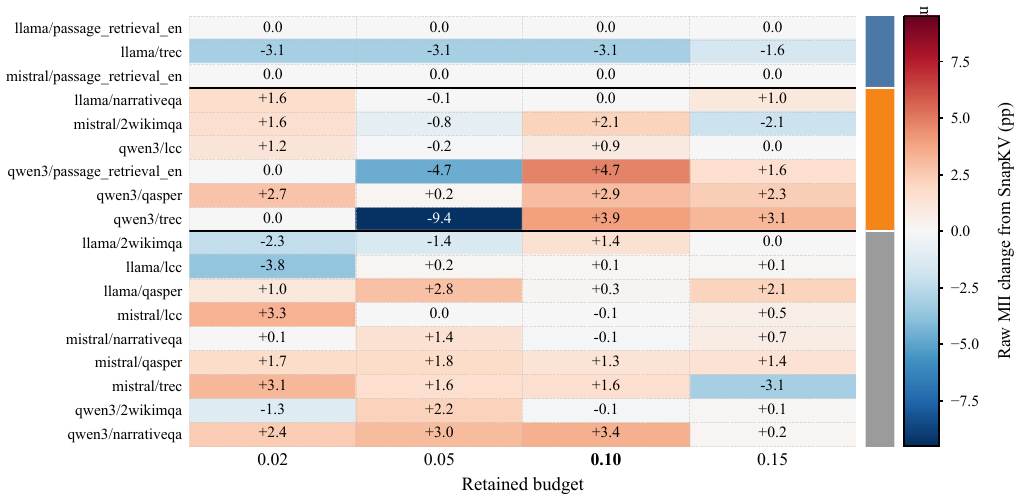}
\caption{Boundary-heavy budget trajectory heatmap.}
\label{fig:budget_phase}
\end{minipage}
\end{nolinenumbers}

\Cref{tab:p30_boundary_summary} compresses the remaining boundary-heavy checks. Raw MII can exceed \fullkv{} when the fixed-contract reference margin is negative, most visibly on Mistral~\citep{jiang2024mistral} TREC. Grouping exact Llama~\citep{grattafiori2024llama}, Mistral, and Qwen3~\citep{yang2025qwen3} by the same task and budget gives 6 of 24 contracts where raw MII is positive for all three models and none where it is negative for all three. Leverage is never uniformly positive; it is uniformly negative in four contracts, uniformly zero in six, and mixed in 14. The portable signal lies in support and value realizability, while leverage remains task-, model-, and budget-conditioned.

\par\vspace{0.10em}
\begin{nolinenumbers}
\noindent\begin{minipage}{\textwidth}
\centering
\captionsetup{type=table,hypcap=false,skip=2pt}
\caption{Boundary-heavy reference-margin and cross-model sign checks.}
\label{tab:p30_boundary_summary}
\scriptsize
\setlength{\tabcolsep}{3pt}
\renewcommand{\arraystretch}{1.04}
\begin{threeparttable}
\begin{tabular*}{\textwidth}{@{\extracolsep{\fill}}lcccccc@{}}
\toprule
\multicolumn{7}{c}{\textbf{Representative FullKV-reference counterexamples}} \\
\midrule
\textbf{Source} & \textbf{Model} & \textbf{Task} & \textbf{Budget} & \textbf{\(m_c\)} & \textbf{MII--SnapKV} & \textbf{MII--FullKV} \\
\midrule
Boundary-heavy & Mistral & trec & 0.02 & -20.31 & +3.12 & +23.44 \\
Boundary-heavy & Mistral & trec & 0.05 & -17.19 & +1.56 & +18.75 \\
Boundary-heavy & Qwen3 & passage\_retrieval\_en & 0.02 & -9.38 & +0.00 & +9.38 \\
Boundary-heavy & Mistral & trec & 0.10 & -4.69 & +1.56 & +6.25 \\
Boundary-heavy & Mistral & trec & 0.15 & -6.25 & -3.12 & +3.12 \\
\bottomrule
\end{tabular*}
\vspace{0.35em}
\begin{tabular*}{\textwidth}{@{\extracolsep{\fill}}lcccccc@{}}
\toprule
\multicolumn{7}{c}{\textbf{Cross-model sign contracts over 24 matched task-budget contracts}} \\
\midrule
\textbf{Signal} & \makecell{\textbf{All}\\\textbf{positive}} & \makecell{\textbf{All}\\\textbf{negative}} & \makecell{\textbf{All}\\\textbf{zero}} & \textbf{Mixed} & \textbf{Mean} & \textbf{Spread} \\
\midrule
Ref. margin & 10 & 0 & 0 & 14 & +2.38 & +7.34 \\
Support & 1 & 0 & 0 & 23 & +0.05 & +2.64 \\
NoLev & 6 & 0 & 1 & 17 & +0.43 & +2.84 \\
MII & 6 & 0 & 1 & 17 & +0.38 & +3.33 \\
Leverage & 0 & 4 & 6 & 14 & -0.05 & +1.28 \\
\bottomrule
\end{tabular*}
\begin{tablenotes}[flushleft]
\footnotesize
\item \(m_c=\fullkv{}_c-\mathrm{Host}_c\), in percentage points; in these SnapKV-contract rows, \(\mathrm{Host}_c=\mathrm{SnapKV}_c\). The upper block lists boundary-heavy rows where raw MII exceeds \fullkv{} despite negative \(m_c\), so \fullkv{} defines the signed reference margin. The lower block groups exact Llama, Mistral, and Qwen3 by the same task and budget.
\end{tablenotes}
\end{threeparttable}
\end{minipage}
\end{nolinenumbers}
\par\vspace{0.10em}

\subsection{Scope and Limitations}
\label{sec:app_additional_controls_scope}

These controls diagnose the existing selector contract. The value-channel ablation tests whether the additive score still works after removing $T^u$; the boundary reference tests whether disagreement boundaries align with a token-distribution proxy. The leverage and budget stress analyses are diagnostic surfaces for budget phase and leverage strength. They stay within the fixed observation-window contract and do not replace the matched LongBench grid.

\section{Limitations}
\label{sec:app_limitations}

The diagnosis is local to a selector contract. Our cleanest claim is under the SnapKV~\citep{li2024snapkv} observation-window contract, where access estimation and projection are fixed and only the ranking scalar is changed. The cross-contract rows in \Cref{tab:baseline_viability} test transfer under contract drift. Once the access window, allocation rule, or projection map changes, the experiment no longer isolates the same stage.

The matched LongBench grid covers three 7B-class instruction models, two budgets, and 100 samples per cell. Coverage at larger scale, longer context, and other benchmark formats remains open.

The diagnostic assumes access to prefill attention and value states. Serving stacks that quantize the cache before selection or use fused kernels without exposing intermediate states need a different implementation route. The paper also restricts attention to non-reconstructing eviction, so quantization, low-rank reconstruction, compensation tokens, and cache offloading need separate fixed-contract tests before the same attribution logic applies.

\section{Broader Impacts}
\label{sec:app_broader_impacts}

KV-cache compression is primarily an efficiency tool. Better cache-selection diagnostics can lower serving cost and energy use, make long-context models easier to run on constrained hardware, and reduce over-provisioning when a smaller retained cache is sufficient. These savings are not guaranteed to reduce aggregate compute demand: cheaper inference can also increase total usage.

The main deployment risk is silent loss of evidence. An eviction rule can remove states needed for faithful generation, especially in multi-document, code, legal, medical, or retrieval-heavy prompts where related facts must survive together. The staged diagnosis separates value-ranking errors from access-support and projection failures; aggressive compression in high-stakes settings still requires task-specific validation.

The experiments use public benchmarks and publicly released pretrained models. They do not collect private data, train new foundation models, or add surveillance, persuasion, or content-generation capabilities. In deployment, cache compression should be paired with task-specific evaluation, conservative fallbacks to larger caches or \fullkv{}, and monitoring for prompts where projection or multi-target support closure becomes the bottleneck.

\par\bigskip

\end{document}